\pdfoutput=1
\documentclass[letterpaper,11pt]{article}
\usepackage{amsmath,amsthm,amssymb,amsfonts}
\usepackage{graphicx, fullpage, color}

\usepackage{url}

\usepackage{epstopdf} 

\usepackage{mathtools}
\usepackage{multirow}
\usepackage{array}
\usepackage[colorinlistoftodos,bordercolor=orange,backgroundcolor=orange!20,linecolor=orange,textsize=scriptsize]{todonotes}

\usepackage{algpseudocode,algorithm}

\usepackage[shortlabels]{enumitem}
\usepackage{xspace}

\DeclareMathOperator*{\argmin}{arg\,min}

\newcommand{\R}{\mathbb{R}}
\newcommand{\Exp}{\mathbb{E}}

\newcommand{\eqdef}{\overset{\text{def}}{=}}

\def\<#1,#2>{\left\langle #1, #2 \right\rangle}

\newcommand{\grad}{\nabla}

\newtheorem{theorem}{Theorem}

\newtheorem{lemma}[theorem]{Lemma}
\newtheorem{proposition}[theorem]{Proposition}

\newtheorem{remark}[theorem]{Remark}

\newcommand{\pp}{\mathcal{P}}
\newcommand{\FedOpt}{Federated Optimization\xspace}
\newcommand{\fedopt}{federated optimization\xspace}
\newcommand{\SVRG}{SVRG\xspace}
\newcommand{\algname}{FSVRG\xspace}  

\newcommand{\node}{node\xspace} 
\newcommand{\nodes}{nodes\xspace}
\newcommand{\iid}{IID\xspace}
\newcommand{\BO}{\mathcal{O}}
\newcommand{\A}{\mathcal{A}}

\usepackage{hyperref}

\title{Federated Optimization: \\ Distributed Machine Learning for On-Device Intelligence}

\author{
Jakub Kone\v{c}n\'{y} \\ University of Edinburgh \\ \texttt{kubo.konecny@gmail.com}
\and H. Brendan McMahan \\ Google \\ \texttt{mcmahan@google.com}
\and Daniel Ramage \\ Google \\ \texttt{dramage@google.com} 
\and Peter Richt\'{a}rik \\ University of Edinburgh \\ \texttt{peter.richtarik@ed.ac.uk}}

\begin{document}

\maketitle
\begin{abstract}
We introduce a new and increasingly relevant setting for distributed optimization in machine learning, where the data defining the optimization are unevenly distributed over an extremely large number of nodes. The  goal is to train a high-quality centralized model. We refer to this setting as \emph{\FedOpt}. In this setting, communication efficiency is of the utmost importance and minimizing the number of rounds of communication is the principal goal.

A motivating example arises when we keep the training data locally on users' mobile devices instead of logging it to a data center for training. In \fedopt, the devices are used as compute nodes performing computation on their local data in order to update a global model. We suppose that we have extremely large number of devices in the network --- as many as the number of users of a given service, each of which has only a tiny fraction of the total data available. In particular, we expect the number of data points available locally to be much smaller than the number of devices. Additionally, since different users generate data with different patterns, it is reasonable to assume that no device has a representative sample of the overall distribution.

We show that existing algorithms are not suitable for this setting, and propose a new algorithm which shows encouraging experimental results for sparse convex problems. This work also sets a path for future research needed in the context of \fedopt.
\end{abstract}

\section{Introduction}
\label{sec:intro}
Mobile phones and tablets are now the primary computing devices for
many people. In many cases, these devices are rarely separated from 
their owners \cite{CNNSmartphoneUsage},
and the combination of rich user interactions and powerful sensors
means they have access to an unprecedented amount of data, much of it
private in nature. Models learned on such data hold the promise of
greatly improving usability by powering more intelligent applications,
but the sensitive nature of the data means there are risks and
responsibilities to storing it in a centralized location.

We advocate an alternative --- {\em federated learning} --- that leaves the training data distributed on the mobile devices, and learns a shared model by aggregating locally computed updates via a central coordinating server. This is a direct application of the principle of focused collection or data minimization proposed by the 2012 White House report on the privacy of consumer data \cite{whitehouse13privacy}. Since these updates are specific to improving the current model, they can be purely ephemeral --- there is no reason to store them on the server once they have been applied. Further, they will never contain more information than the raw training data (by the data processing inequality), and will generally contain much less. A principal advantage of this approach is the decoupling of model training from the need for direct access to the raw training data. Clearly, some trust of the server coordinating the training is still required, and depending on the details of the model and algorithm, the updates may still contain private information. However, for applications where the training objective can be specified on the basis of data available on each client, federated learning can significantly reduce privacy and security risks by limiting the attack surface to only the device, rather than the device and the cloud.

If additional privacy is needed, randomization techniques from differential privacy can be used.  The centralized algorithm could be modified to produce a differentially private model~\cite{chaudhuri11dperm, dwork14book, abadi2016deep}, which allows the model to be released while protecting the privacy of the individuals contributing updates to the training process.  If protection from even a malicious (or compromised) coordinating server is needed, techniques from local differential privacy can be applied to privatize the individual updates \cite{duchi14privacy}.  Details of this are beyond the scope of the current work, but it is a promising direction for future research.

A more complete discussion of applications of federated learning as
well as privacy ramifications can be found
in~\cite{mcmahan2016federated}. Our focus in this work will be on \fedopt, the
optimization problem that must be solved in order to make federated
learning a practical alternative to current approaches.

\subsection{Problem Formulation}
The optimization community has seen an explosion of interest in solving problems with finite-sum structure in recent years. In general, the objective is formulated as
\begin{equation}
\label{eq:problem}
\min_{w \in \R^d} f(w) \qquad \text{where} \qquad f(w) \eqdef \frac{1}{n} \sum_{i=1}^n f_i(w).
\end{equation}
The main source of motivation are problems arising in machine learning. The problem structure~\eqref{eq:problem} covers linear or logistic regressions, support vector machines, but also more complicated models such as conditional random fields or neural networks.

We suppose we have a set of input-output pairs $\{ x_i, y_i \}_{i = 1}^n$, and a loss function, giving rise to the functions $f_i$. Typically, $x_i \in \R^d$ and $y_i \in \R$ or $y_i \in \{ -1, 1 \}$. Simple examples include 
\begin{itemize}
\item linear regression:  $f_i(w) = \frac12 (x_i^Tw - y_i)^2$, $y_i \in \R$
\item logistic regression: $f_i(w) = -\log(1 + \exp(-y_i x_i^T w))$, $y_i \in \{-1, 1\}$ 
\item support vector machines: $f_i(w) = \max\{0, 1 - y_i x_i^T w \}$, $y_i \in \{-1, 1\}$
\end{itemize}

More complicated non-convex problems arise in the context of neural networks, where rather than via the linear-in-the-features mapping $x_i^T w$, the network makes prediction through a non-convex function of the feature vector $x_i$. However, the resulting loss can still be written as $f_i(w)$, and gradients can be computed efficiently using backpropagation.

The amount of data that businesses, governments and academic projects collect is rapidly increasing. Consequently, solving problem~\eqref{eq:problem} arising in practice is often impossible on a single \node, as merely storing the whole dataset on a single \node becomes infeasible. This necessitates the use of a distributed computational framework, in which the training data describing the problem is stored in a distributed fashion across a number of interconnected \nodes and the optimization problem is solved collectively by the cluster of nodes.

Loosely speaking, one can use any network of \nodes to simulate a single powerful \node, on which one can run any algorithm. The
practical issue is that the time it takes to communicate between a
processor and memory on the same \node is normally many
orders of magnitude smaller than the time needed for two \nodes to
communicate; similar conclusions hold for the energy required
\cite{shalf2011exascale}. Further, in order to take advantage of parallel
computing power on each node, it is necessary to subdivide the problem
into subproblems suitable for independent/parallel computation.

State-of-the-art optimization algorithms are typically inherently sequential. Moreover, they usually rely on performing a large number of very fast iterations. The problem stems from the fact that if one needs to perform a round of communication after each iteration, practical performance drops down dramatically, as the round of communication is much more time-consuming than a single iteration of the algorithm.

These considerations have lead to the development of novel algorithms specialized for distributed optimization (we defer thorough review until  Section~\ref{sec:relatedWork}). For now, we note that most of the results in literature work in the setting where the data is evenly distributed, and further suppose that $K \ll n / K$ where $K$ is the number of \nodes. This is indeed often close to reality when data is stored in a large data center. Additionally, an important subfield of the field of distributed learning relies on the assumption  that each machine has a representative sample of the data available locally. That is, it is assumed that each machine has an \iid sample from the underlying distribution. However, this assumption is often too strong; in fact,  even in the data center paradigm this is often not the case since the data on a single \node can be close to each other on a temporal scale, or clustered by its geographical origin. Since the patterns in the data can change over time, a feature might be present frequently on one \node, while not appear on another at all.

The \fedopt setting describes a novel optimization scenario where none of the above assumptions  hold. We outline this setting in more detail in the following section.

\subsection{The Setting of \FedOpt}
\label{sec:intro:challenge}

The main purpose of this paper is to bring to the attention of the machine learning and optimization communities a new and increasingly practically relevant setting for distributed optimization, where none of the typical assumptions are satisfied, and communication efficiency is of utmost importance. In particular, algorithms for \fedopt must handle training data with the following characteristics:
\begin{itemize}
\item \textbf{Massively Distributed}: Data points are stored across a large number of \nodes $K$. In particular, the number of \nodes can be much bigger than the average number of training examples stored on a given \node ($n/K$).
\item \textbf{Non-\iid}: Data on each \node may be drawn from a different distribution; that is, the data points available locally are far from being a representative sample of the overall distribution.
\item \textbf{Unbalanced}: Different \nodes may vary by orders of magnitude in the number of training examples they hold.
\end{itemize}

In this work, we are particularly concerned with \textbf{sparse} data, where some features occur on  a small subset of nodes or data points only. Although this is not necessary characteristic of the setting of \fedopt, we will show that the sparsity structure can be used to develop an effective algorithm for \fedopt. Note that data arising in the largest machine learning problems being solved nowadays, ad click-through rate predictions, are extremely sparse.

We are particularly interested in the setting where training data lives on users' mobile devices (phones and tablets), and the data may be privacy sensitive. The data $\{x_i, y_i\}$ is generated through device usage, e.g., via interaction with apps. Examples include predicting the next word a user will type (language modelling for smarter keyboard apps), predicting which photos a user is most likely to share, or predicting which notifications are most important. 

To train such models using traditional distributed learning algorithms, one would collect the training examples in a centralized location (data center) where it could be shuffled and distributed evenly over proprietary compute nodes.  In this paper we propose and study an alternative model: the training examples are not sent to a centralized location, potentially saving  significant network bandwidth and providing additional privacy protection. In exchange, users allow some use of their devices' computing power, which shall be used to train the model.

In the communication model of this paper, in each round we send an update
$\delta \in \R^d$ to a centralized server, where $d$ is the dimension of the model
being computed/improved. The update $\delta$ could be a gradient vector, for
example.  While it is certainly possible that in some applications the
$\delta$ may encode some private information of the user, it is likely
much less sensitive (and orders of magnitude smaller) than the
original data itself. For example, consider the case where the raw
training data is a large collection of video files on a mobile
device.  The size of the update $\delta$ will be \emph{independent} of
the size of this local training data corpus. 
We show that a global model can be trained using a small number of
communication rounds, and so this also reduces the network
bandwidth needed for training by orders of magnitude compared to
copying the data to the datacenter.

Further, informally, we choose $\delta$ to be the minimum piece of
information necessary to improve the global model; its utility for
other uses is significantly reduced compared to the original
data. Thus, it is natural to design a system that does not store these
$\delta$'s longer than necessary to update the model, again increasing
privacy and reducing liability on the part of the centralized model
trainer. This setting, in which a single vector $\delta \in \R^d$ is communicated in each round, covers most existing first-order methods, including dual methods such as CoCoA+ \cite{ma2015distributed}.

Communication constraints arise naturally in the massively distributed setting, as network connectivity may be limited (e.g., we may wish to deffer all communication until the mobile device is charging and connected to a wi-fi network).  Thus, in realistic scenarios we may be limited to only a single round of communication per day. This implies that, within reasonable bounds, we have access to essentially unlimited local computational power. Consequently, the practical objective is solely to minimize the number of  communication rounds.

The main purpose of this work is initiate research into, and design a first practical implementation of \fedopt. Our results suggest that with suitable optimization algorithms, very little is lost by not having an \iid sample of the data available, and that even in the presence of a large number of \nodes, we can still achieve convergence in relatively few rounds of communication.

\section{Related Work}
\label{sec:relatedWork}

In this section we provide a detailed overview of the relevant literature. We particularly focus on algorithms that can be used to solve  problem~\eqref{eq:problem} in various contexts. First, in Sections~\ref{sec:relatedWork:general} and \ref{sec:oh0s9hs} we look at algorithms designed to be run on a single computer. In Section~\ref{sec:relatedWork:distributedSetting} we follow with a discussion of the distributed setting, where no single \node has direct access to all data describing $f$. We describe a paradigm for measuring the efficiency of distributed methods, followed by overview of existing methods and commentary  on whether they were designed with communication efficiency in mind or not.

\subsection{Baseline Algorithms}
\label{sec:relatedWork:general}

In this section we shall describe several fundamental baseline algorithms which can be used to solve problems of the form \eqref{eq:problem}.

\paragraph{Gradient Descent.} A trivial benchmark for solving  problems of structure~\eqref{eq:problem} is {\em Gradient Descent} (GD) in the case when functions $f_i$ are smooth (or Subgradient Descent for non-smooth functions) \cite{Nesterov-book}. The GD algorithm performs the iteration \[w^{t+1} = w^t - h_t \grad f(w^t),\] where $h_t>0$ is a  stepsize parameter. As we mentioned earlier, the number of functions, or equivalently, the number of training data pairs, $n$,  is typically very large. This makes GD impractical, as it needs to process the whole dataset in order to evaluate a single gradient and update the model. 

Gradient descent can be substantially accelerated, in theory and practice, via the addition of a momentum term. Acceleration ideas for gradient methods in convex optimization can be traced back to the work of Polyak \cite{Polyak-heavy-ball} and Nesterov \cite{Nesterov-1983, Nesterov-book}. While accelerated GD methods have a substantially better convergence rate, in each iteration they still need to do at least one pass over all data. As a result, they are not practical for problems where $n$ very large.

\paragraph{Stochastic Gradient Descent.} At present a basic, albeit in practice extremely popular, alternative to GD is { \em Stochastic Gradient Descent} (SGD), dating back to the seminal work of Robbins and Monro \cite{robbinsMonro}. In the context of~\eqref{eq:problem}, SGD samples a random function (i.e., a random data-label pair) $i_t \in \{1, 2, \dots, n\}$ in iteration $t$, and performs the update \[w^{t+1} = w^t - h_t \grad f_{i_t}(w^t),\] where $h_t>0$ is a stepsize parameter. Intuitively speaking, this method works because if $i_t$ is sampled uniformly at random from indices $1$ to $n$, the update direction is an unbiased estimate of the gradient --- $\Exp[\grad f_{i_t}(w)] = \grad f(w)$. However, noise introduced by sampling slows down the convergence, and a diminsihing sequence of stepsizes $h_k$ is necessary for convergence. For a theoretical analysis for convex functions we refer the reader to \cite{nemirovski2009robust, moulines2011bach, needell2014stochastic} and \cite{pegasos, MB2013} for SVM problems. In a recent review \cite{bottou2016optimization}, the authors outline further research directions. For a more practically-focused discussion, see \cite{bottou2012SGDtricks}. In the context of neural networks, computation of stochastic gradients is referred to as \emph{backpropagation} \cite{lecun2012efficientbackprop}. Instead of specifying the functions $f_i$ and its gradients explicitly, backpropagation is a general way of computing the gradient. Performance of several competitive algorithms for training deep neural networks has been compared in \cite{ngiam2011optimization}.

One common trick that has been practically observed to provide superior performance, is to replace random sampling in each iteration by going through all the functions in a random order. This ordering is replaced by another random order after each such cycle \cite{bottou2009curiously}. Theoretical understanding of this phenomenon had been a long standing open problem, understood recently in \cite{gurbuzbalaban2015randomreshuffling}.

The core differences between GD and SGD can be summarized as follows. GD has a fast convergence rate, but each iteration in the context of \eqref{eq:problem} is potentially very slow, as it needs to process the entire dataset in each iteration. On the other hand, SGD has slower convergence rate, but each iteration is fast, as the work needed is independent of number of data points $n$. For the problem structure of \eqref{eq:problem}, SGD is usually better, as for practical purposes relatively low accuracy is required, which SGD can in extreme cases achieve after single pass through data, while GD would make just a single update. However, if a high accuracy was needed, GD or its faster variants would prevail.

\subsection{A Novel Breed of Randomized Algorithms}\label{sec:oh0s9hs}

Recent years have seen an explosion of new randomized methods which, in a first approximation, combine the benefits of cheap iterations of SGD with fast convergence of GD. Most of these methods can be said to belong to one of two classes --- dual methods of the randomized coordinate descent variety, and primal methods of the stochastic gradient descent with variance reduction variety.

\paragraph{Randomized Coordinate Descent.}  Although the idea of coordinate descent has been around for several decades in various contexts (and for quadratic functions dates back even much further, to works on the Gauss-Seidel methods), it came to prominence in machine learning and optimization with the work of  Nesterov \cite{nesterovCD} which equipped the method with a randomization strategy. Nesterov's work on {\em Randomized Coordinate Descent} (RCD) popularized the method and demonstrated that randomization can be very useful for problems of structure \eqref{eq:problem}.

The RCD algorithm in each iteration chooses a random coordinate $j_t \in \{ 1, \dots, d\}$ and performs the update \[w^{t+1} = w^t - h_{j_t} \grad_{j_t} f(w^t) e_{j_t},\] where $h_{j_t}>0$ is a stepsize parameter, $\grad_j f(w)$ denotes the $j^{th}$ partial derivative of function $f$, and $e_j$ is the $j^{th}$ unit standard basis vector in $\R^d$. For the case of generalized linear models, when the data exhibits certain sparsity structure, it is possible to evaluate the partial derivative $\grad_j f(w)$ efficiently, i.e., without need to process the entire dataset, leading to a practically efficient algorithm, see for instance  \cite[Section 6]{richtarikCD}.

Numerous follow-up works extended the concept to proximal setting \cite{richtarikCD}, single processor parallelism \cite{bradley2011PCDM, richtarikCDparallel} and develop efficiently implementable acceleration \cite{leeSidfordCD}. All of these three properties were connected in a single algorithm in \cite{approx}, to which we refer the reader for a review of the early developments in the area of RCD, particularly to overview in Table~1 therein.

\paragraph{Stochastic Dual Coordinate Ascent.} When an explicit strongly convex, but not necessarily smooth, regularizer is added to the average loss \eqref{eq:problem}, it is possible to write down its (Fenchel) dual and the dual variables live in $n$-dimensional space. Applying RCD leads to an algorithm for solving \eqref{eq:problem} known under the name {\em Stochastic Dual Coordinate Ascent} \cite{SDCA}. This method has gained broad popularity with practicioners, likely due to the fact that for a number of loss functions, the method comes without the need to tune any hyper-parameters. The work \cite{SDCA} was first to show that by applying RCD  \cite{richtarikCD} to the dual problem, one also solves the primal problem \eqref{eq:problem}. For a theoretical and computational comparison of applying RCD to the primal versus the dual problems, see \cite{faceoff}.

A directly primal-dual randomized coordinate descent method called Quartz, was developed in \cite{Quartz}. It has been recently shown in SDNA \cite{SDNA} that incorporating curvature information contained in random low dimensional subspaces spanned by a few coordinates can sometimes lead to dramatic speedups. Recent works \cite{dfSDCAupdate, dfSDCACsiba} interpret the SDCA method in primal-only setting, shedding light onto why this method works as a SGD method with a version of variance reduction property. 

\bigskip
We now move the the second class of novel randomized algorithms which can be generally interpreted as variants of SGD, with an attempt to reduce variance inherent in the process of gradient estimation.

\paragraph{Stochastic Average Gradient.} The first notable algorithm from this class is the {\em Stochastic Average Gradient} (SAG) \cite{SAGNIPS, SAG}. The SAG algorithm stores an average of $n$ gradients of functions $f_i$ evalueated at different points in the history of the algorithm. In each iteration, the algotithm, updates randomly chosen gradient out of this average, and makes a step in the direction of the average. This way, complexity of each iteration is independent of $n$, and the algorithm enjoys a fast convergence. The drawback of this algorithm is that it needs to store $n$ gradients in memory because of the update operation. In the case of generalized linear models, this memory requirement can be reduced to the need of $n$ scalars, as the gradient is a scalar multiple of the data point. This methods has been recently extended for use in Conditional Random Fields \cite{SAGCRF}. Nevertheless, the memory requirement makes the algorithm infeasible for application even in relatively small neural networks.

A followup algorithm SAGA \cite{SAGA} and its simplification \cite{defazio2016simple}, modifies the SAG algorithm to achieve unbiased estimate of the gradients. The memory requirement is still present, but the method significantly simplifies theoretical analysis, and yields a slightly stronger convergence guarantee.

\paragraph{Stochastic Variance Reduced Gradient.} Another algorithm from the SGD class of methods is {\em Stochastic Variance Reduced Gradient}\footnote{The same algorithm was simultaneously introduced as Semi-Stochastic Gradient Descent (S2GD) \cite{S2GD}. Since the former work gained more attention, we will for clarity use the name SVRG throughout this paper.} (SVRG) \cite{SVRG} and \cite{S2GD, proxSVRG, mS2GD}. The SVRG algorithm runs in two nested loops. In the outer loop, it computes full gradient of the whole function, $\grad f(w^t)$, the expensive operation one tries to avoid in general. In the inner loop, the update step is iteratively computed as \[w = w - h [\grad f_i(w) - \grad f_i(w^t) + \grad f(w^t)].\] The core idea is that the stochastic gradients are used to estimate the change of the gradient between point $w^t$ and $w$, as opposed to estimating the gradient directly. We return to more detailed description of this algorithm in Section~\ref{sec:SVRG}.

The SVRG has the advantage that it does not have the additional memory requirements of SAG/SAGA, but it needs to process the whole dataset every now and then. Indeed, comparing to SGD, which typically makes significant progress in the first pass through data, SVRG does not make any update whatsoever, as it needs to compute the full gradient. This and several other practical issues have been recently addressed in \cite{practicalSVRG}, making the algorithm competitive with SGD early on, and superior in later iterations. Although there is nothing that prevents one from applying SVRG and its variants in deep learning, we are not aware of any systematic assessment of its performance in this setting. Vanilla experiments in \cite{SVRG, reddi2016stochastic} suggest that SVRG matches basic SGD, and even outperforms in the sense that variance of the iterates seems to be significantly smaller for SVRG. However, in order to draw any meaningful conclusions, one would need to perform extensive experiments and compare with state-of-the-art methods usually equipped with numerous heuristics.

There already exist attempts at combining SVRG type algorithms with randomized coordinate descent \cite{S2CD, wang2014randomized}. Although these works highlight some interesting theoretical properties, the algorithms do not seem to be practical at the moment; more work is needed in this area. The first attempt to unify algorithms such as SVRG and SAG/SAGA already appeared in the SAGA paper \cite{SAGA}, where the authors interpret SAGA as a midpoint between SAG and SVRG. Recent work \cite{reddi2015variance} presents a general algorithm, which recovers SVRG, SAGA, SAG and GD as special cases, and obtains an asynchronous variant of these algorithms as a byproduct of the formulation. SVRG can be equipped with momentum (and negative momentum), leading to a new accelerated SVRG method known as Katyusha \cite{Katyusha}. SVRG can be further accelerated via a raw clustering mechanism \cite{SVRG-rawclusters}.

\paragraph{Stochastic Quasi-Newton Methods.} A third class of new algorithms are the {\em Stochastic quasi-Newton} methods \cite{stochasticQN, bordes2009sgd}. These algorithms in general try to mimic the limited memory BFGS method (L-BFGS) \cite{LBFGS}, but model the local curvature information using inexact gradients --- coming from the SGD procedure. A recent attempt at combining these methods with SVRG can be found in \cite{moritz2015linearly}. In \cite{SBFGS}, the authors utilize recent progress in the area of stochastic matrix inversion \cite{inverse} revealing new connections with quasi-Newton methods, and devise a new  stochastic limited memory BFGS method working in tandem  with SVRG.  The fact that the theoretical understanding of this branch of research is the least understood and having several details making the implementation more difficult compared to the methods above may limit its wider use. However, this approach could be most promising for deep learning once understood better.

One important aspect of machine learning is that the Empirical Risk Minimization problem \eqref{eq:problem} we are solving is just a proxy for the Expected Risk we are ultimately interested in. When one can find exact minimum of the empirical risk, everything reduces to balancing approximation--estimation tradeoff that is the object of abundant literature --- see for instance \cite{vapnik1999overview}. An assessment of asymptotic performance of some optimization algorithms as \emph{learning} algorithms in large-scale learning problems\footnote{See \cite[Section 2.3]{BottouBousquet} for their definition of large scale learning problem.} has been introduced in \cite{BottouBousquet}. Recent extension in \cite{practicalSVRG} has shown that the variance reduced algorithms (SAG, SVRG, \dots) can in certain setting be better \emph{learning} algorithms than SGD, not just better optimization algorithms.

\paragraph{Further Remarks.} A general method, referred to as Universal Catalyst \cite{lin2015universal, frostig2015regularizing}, effectively enables conversion of a number of the algorithms mentioned in the previous sections to their `accelerated' variants. The resulting convergence guarantees nearly match lower bounds in a number of cases. However, the need to tune additional parameter makes the method rather impractical.

Recently, lower and upper bounds for complexity of stochastic methods on problems of the form \eqref{eq:problem} were recently obtained in \cite{Srebro-lower_and_upper2016}.

\subsection{Distributed Setting}
\label{sec:relatedWork:distributedSetting}

In this section we review the literature concerning algorithms for solving \eqref{eq:problem} in the distributed setting. When we speak about distributed setting, we refer to the case when the data describing the functions $f_i$ are not stored on any single storage device. This can include setting where one's data just don't fit into a single RAM/computer/node, but two is enough. This also covers the case where data are distributed across several datacenters around the world, and across many \nodes in those datacenters. The point is that in the system, there is no single processing unit that would have direct access to all the data. Thus, the distributed setting does not include single processor parallelism\footnote{It should be noted that some of the works presented in this section were originally presented as parallel algorithms. We include them anyway as many of the general ideas in carry over to the distributed setting.}. Compared with local computation on any single \node, the cost of communication between \nodes is much higher both in terms of speed and energy consumption \cite{bekkerman2011scaling, shalf2011exascale}, introducing new computational challenges, not only for optimization procedures.

We first reveiew a theoretical decision rule for determining the practically best algorithm for a given problem in Section~\ref{sec:relatedWork:paradigm}, followed by overview of distributed algorithms in Section~\ref{sec:relatedWork:distributedAlgorithms}, and communication efficient algorithms in Section~\ref{sec:relatedWork:commEffAlgs}. The following paradigm highlights why the class of communication efficient algorithms are not only preferable choice in the trivial sense. The communication efficient algorithms provide us with much more flexible tools for designing overall optimization procedure, which can make the algorithms inherently adaptive to differences in computing resources and architectures.

\subsubsection{A Paradigm for Measuring Distributed Optimization Efficiency}
\label{sec:relatedWork:paradigm}

This section reviews a paradigm for comparing efficency of distributed algorithms. Let us suppose we have many algorithms $\A$ readily available to solve the problem~\eqref{eq:problem}. The question is: ``How do we decide which algorithm is the best for our purpose?'' Initial version of this reasoning already appeared in \cite{ma2015distributed}, and applies also to \cite{AIDE}.

First, consider the basic setting on a single machine. Let us define $\mathcal{I}_\A(\epsilon)$ as the number of iterations algorithm $\A$ needs to converge to some fixed $\epsilon$ accuracy. Let $\mathcal{T}_\A$ be the time needed for a single iteration. Then, in practice, the best algorithm is one that minimizes the following quantity.\footnote{Considering only algorithms that can be run on a given machine.}
\begin{equation}
\label{eq:paradigmBasic}
\text{TIME} = \mathcal{I}_\A(\epsilon) \times \mathcal{T}_\A.
\end{equation}

The number of iterations $\mathcal{I}_\A(\epsilon)$ is usually given by theoretical guarantees or observed from experience. The $\mathcal{T}_\A$ can be empirically observed, or one can have idea of how the time needed per iteration varies between different algorithms in question. The main point of this simplified setting is to highlight key issue with extending algorithms to the distributed setting.

The natural extension to distributed setting is the following. Let $c$ be time needed for communication during a single iteration of the algorithm $\A$. For sake of clarity, we suppose we consider only algorithms that need to communicate a single vector in $\R^d$ per round of communication. Note that essentially all first-order algorithms fall into this category, so this is not a restrictive assumption, which effectively sets $c$ to be a constant, given any particular distributed architecture one has at disposal.

\begin{equation}
\label{eq:paradigm}
\text{TIME} = \mathcal{I}_\A(\epsilon) \times (c + \mathcal{T}_\A)
\end{equation}

The communication cost $c$ does not only consist of actual exchange of the data, but also many other things like setting up and closing a connection between \nodes. Consequently, even if we need to communicate very small amount of information, $c$ always remains above a nontrivial threshold.

Most, if not all, of the current state-of-the-art algorithms that are the best in setting of~\eqref{eq:paradigmBasic}, are stochastic and rely on doing very large number (big $\mathcal{I}_\A(\epsilon)$) of very fast (small $\mathcal{T}_\A$) iterations. As a result, even relatively small $c$ can cause the practical performance of those algorithms drop down dramatically, because $c \gg \mathcal{T}_\A$.

This has been indeed observed in practice, and motivated development of new methods, designed with this fact in mind from scratch, which we review in Section~\ref{sec:relatedWork:distributedAlgorithms}. Although this is a good development for academia --- motivation to explore new setting, it is not necessarily a good news for the industry.

Many companies have spent significant resources to build excellent algorithms to tackle their problems of form~\eqref{eq:problem}, fine tuned to the specific patterns arising in their data and side applications required. When the data companies collect grows too large to be processed on a single machine, it is understandable that they would be reluctant to throw away their fine tuned algorithms. This issue was first time explicitly addressed in CoCoA~\cite{ma2015distributed}, which is rather framework than a algorithm, which works as follows (more detailed description follows in Section~\ref{sec:relatedWork:commEffAlgs}).

The CoCoA framework formulates a general way to form a specific subproblem on each \node, based on data available locally and a single shared vector that needs to be distributed to all \nodes. Within a iteration of the framework, each \node uses \emph{any} optimization algorithm $\A$, to reach a relative $\Theta$ accuracy on the local subproblem. Updates from all \nodes are then aggregated to form an update to the global model.

The efficiency paradigm changes as follows:

\begin{equation}
\label{eq:paradigmNew}
\text{TIME} = \mathcal{I}(\epsilon, \Theta) \times (c + \mathcal{T}_\A(\Theta))
\end{equation}

The number of iterations $\mathcal{I}(\epsilon, \Theta)$ is independent of choice of the algorithm $\A$ used as a local solver, because there is theory predicting how many iterations of the CoCoA framework are needed to achieve $\epsilon$ accuracy, if we solve the local subproblems to relative $\Theta$ accuracy. Here, $\Theta = 0$ would mean we require the subproblem to be solved to optimality, and $\Theta = 1$ that we don't need any progress whatsoever. The general upper bound on number of iterations of the CoCoA framework is $\mathcal{I}(\epsilon, \Theta) = \frac{\mathcal{O}(\log(1/\epsilon))}{1 - \Theta}$ \cite{cocoaNIPS, cocoaICML, ma2015distributed} for strongly convex objectives. From the inverse dependence on $1 - \Theta$, we can see that there is a fundamental limit to the number of communication rounds needed. Hence, it will probably not be efficient to spend excessive resources to attain very high local accuracy (small $\Theta$). Time per iteration $\mathcal{T}_\A(\Theta)$ denotes the time algorithm $\A$ needs to reach the relative $\Theta$ accuracy on the local subproblem.

This efficiency paradigm is more powerful for a number of reasons.
\begin{enumerate}
\item It allows practicioners to continue using their fine-tuned solvers, that can run only on single machine, instead of having to implement completely new algorithms from scratch.

\item The actual performance in terms of number of rounds of communication is independent from the choice of optimization algorithm, making it much easier to optimize the overall performance.

\item Since the constant $c$ is architecture dependent, running optimal algorithm on one \node network does not have to be optimal on another. In the setting~\eqref{eq:paradigm}, this could mean moving from one cluster to another, a completely different algorithm is optimal, which is a major change. In the setting~\eqref{eq:paradigmNew}, this can be improved by simply changing $\Theta$, which is typically implicitly determined by number of iterations algorithm $\A$ runs for.
\end{enumerate}

In this work we propose a different way to formulate the local subproblems, which does not rely on duality as in the case of CoCoA. We also highlight that some algorithms seem to be particularly suitable to solve those local subproblems, effectively leading to novel algorithms for distributed optimization.

\subsubsection{Distributed Algorithms}
\label{sec:relatedWork:distributedAlgorithms}

As discussed below in Section \ref{sec:relatedWork:paradigm}, this setting creates unique challenges. Distributed optimization algorithms typically require a small number (1--4) of communication rounds per iteration. By communication round we typically understand a single MapReduce operation \cite{dean2008mapreduce}, implemented efficiently for iterative procedures \cite{MPI}, such as optimization algorithms. Spark \cite{zaharia2010spark} has been established as a popular open source framework for implementing distributed iterative algorithms, and includes several of the algorithms mentioned in this section.

Optimization in distributed setting has been studied for decades, tracing back to at least works of Bertsekas and Tsitsiklis \cite{bertsekas1989parallel, bertsekas1983distributed, tsitsiklis1984problems}. Recent decade has seen an explosion of interest in this area, greatly motivated by rapid increase of data availability in machine learning applications.

Much of the recent effort was focused on creating new optimization algorithms, by building variants of popular algorithms suitable for running on a single processor (See Section~\ref{sec:relatedWork:general}). A relatively common feature of many of these efforts is a) The computation overhead in the case of synchronous algorithms, and b) The difficulty of analysing asynchronous algorithms without restrictive assumptions. By computation overhead we mean that if optimization program runs in a compute-communicate-update cycle, the update part cannot start until all \nodes finish their computation. This causes some of the \nodes be idle, while remaining \nodes finish their part of computation, clearly an inefficient use of computational resources. This pattern often diminishes or completely reverts potential speed-ups from distributed computation. In the asynchronous setting in general, an update can be applied to a parameter vector, followed by computation done based on a now-outdated version of that parameter vector. Formally grasping this pattern, while keeping the setting realistic is often quite challenging. Consequently, this is very open area, and optimal choice of algorithm in any particular case is often heavily dependent on the problem size, details in its structure, computing architecture available, and above all, expertise of the practitioner.

This general issue is best exhibited with numerous attempts at parallelizing the Stochastic Gradient Descent and its variants. As an example, \cite{dekel2012optimal, duchi2012dual} provide theoretically linear speedup with number of \nodes, but are difficult to implement efficiently, as the \nodes need to synchronize frequently in order to compute reasonable gradient averages. As an alternative, no synchronization between workers is assumed in \cite{recht2011hogwild, agarwal2011distributed, duchi2013estimation}. Consequently, each worker reads $w^t$ from memory, parameter vector $w$ at time point $t$, computes a stochastic gradient $\grad f_i(w^t)$ and applies it to already changed state of the parameter vector $w^{t+\tau}$. The above mentioned methods assume that the delay $\tau$ is bounded by a constant, which is not necessarily realistic assumption\footnote{A bound on the delay $\tau$ can be deterministic or probabilistic. However, in practice, the delays are mostly about the number of \nodes in the network, and there rare very long delays, when a variety of operating system-related events can temporarily postpone computation of a single \node. To the best of our knowledge, no formal assumptions reflect this setting well. In fact, two recent works \cite{mania2015perturbed, leblond2016asaga} highlight subtle but important issue with labelling of iterates in the presence of asynchrony, rendering most of the existing analyses of asynchronous optimization algorithms incorrect.}. Some of the works also introduce assumptions on the sparsity structures or conditioning of the Hessian of $f$. Asymptotically optimal convergent rates were proven in \cite{duchi2015asynchronous} with considerably milder assumptions. Improved analysis of asynchronous SGD was also presented in \cite{de2015taming}, simultaneously with a version that uses lower-precision arithmetic was introduced without sacrificing performance, which is a trend that might find use in other parts of machine learning in the following years.

The negative effect of asynchronous distributed implementations of SGD seem to be negligible, when applied to the task of training very large deep networks --- which is the ultimate industrial application of today. The practical usefulness has been demonstrated for instance by Google's Downpour SGD \cite{largeNN} and Microsoft's Project Adam \cite{chilimbi2014project}.

The first distributed versions of Coordinate Descent algorithms were the Hydra and its accelerated variant, Hydra$^2$, \cite{hydra, hydra2}, which has been demonstrated to be very efficient on large sparse problems implemented on a computing cluster. An extended version with description of implementation details is presented in \cite{marecek2014distributed}. Effect of asynchrony has been explored and partially theoretically understood in the works of \cite{liu2013asynchronous, liu2015asynchronous}. Another asynchronous, rather framework than an algorithm, for coordinate updates, applicable to wider class of objectives is presented in \cite{peng2015arock}.

The data are assumed to be partitioned to \nodes by features/coordinates in the above algorithms. This setting can be restrictive if one is not able to distribute the data beforehand, but instead the data are distributed ``as is'' --- in which case the data are most commonly distributed by data points. This does not need to be an issue, if a dual version of coordinate descent is used --- in which the distribution is done by data points \cite{takac2015distributed} followed by works on Communication Efficient Dual Cooridante Ascent, described in next section. The use of duality however requires usage of additional explicit strongly convex regularization term, hence can be used to solve smaller class of problems. Despite the apparent practical disadvantages, variants of distributed coordinate descent algorithms are among the most widely used methods in practice.

Moving to variance reduced methods, distributed versions of SAG/SAGA algorithms have not been proposed yet. However, several distributed versions of the SVRG algorithm already exist. A scheme for replicating data to simulate iid sampling in distributed environment was proposed in \cite{lee2015distributed}. Although the performance is well analysed, the setting requires significantly stronger control of data distribution which is rarely practicaly feasible. A relatively similar method to Algorithm~\ref{alg:DS2GDnaive} presented here has been proposed in \cite{AIDE}, which was analysed, and in \cite{mahajan2015efficient}, a largely experimental work that can be also cast as communication efficient --- described in detail in Section~\ref{sec:relatedWork:commEffAlgs}.

Another class of algorithms relevant for this work is Alternating Direction Method of Multipliers (ADMM) \cite{boyd2011distributed, deng2016global}. These algorithms are in general applicable to much broader class of problems, and hasn't been observed to perform better than other algorithms presented in this section, in the machine learning setting of \eqref{eq:problem}.

\subsubsection{Communication-Efficient Algorithms}
\label{sec:relatedWork:commEffAlgs}

In this Section, we describe algorithms that can be cast as ``communication efficient''. The common theme of the algorithms presented here, is that in order to perform better in the sense of \eqref{eq:paradigm}, one should design algorithms with high $\mathcal{T}_\A$, in order to make the cost of communcation $c$ negligible.

Before moving onto specific methods, it is worth the noting some of the core limits concerning the problem we are trying to solve in distributed setting. Fundamental limitations of stochastic versions of the problem \eqref{eq:problem} in terms of runtime, communication costs and number of samples used are studied in \cite{shamir2014distributed}. Efficient algorithms and lower bounds for distributed statistical estimation are established in \cite{JMLR:v14:zhang13b, zhang2013information}. 

However, these works do not fit into our framework, because they assume that each \node has access to data generated \iid from a single distribution. In the case of \cite{JMLR:v14:zhang13b, zhang2013information} also $K \ll n / K$, that the number of \nodes $K$ is much smaller than the number of data point on each \node is also assumed. As we stress in the Introduction, these assumptions are far from being satisfied in our setting. Intuitively, relaxing these assumptions should make the problem harder. However, it is not as straightforward to conclude this, as there are certainly particular non-iid data distributions that simplify the problem --- for instance if data are distributed according to separability structure of the objective. Lower bounds on communication complexity of distributed convex optimization of \eqref{eq:problem} are presented in \cite{CommShamir}, concluding that for \iid data distributions, existing algorithms already achieve optimal complexity in specific settings. 

Probably first, rather extreme, work \cite{zinkevich2010parallelized} proposed to parallelize SGD in a single round of communication. Each node simply runs SGD on the data available locally, and their outputs are averaged to form a final result. This approach is however not very robust to differences in data distributions available locally, and it has been shown \cite[Appendix A]{DANE} that in general it cannot perform better than using output of a single machine, ignoring all the other data.

Shamir et al. proposed the DANE algorithm, Distributed Approximate Newton \cite{DANE}, to exactly solve a general subproblem available locally, before averaging their solutions. The method relies on similarity of Hessians of local objectives, representing their iterations as an average of inexact Newton steps. We describe the algorithm in greater detail in Section~\ref{sec:algorithms:DANE} as our proposed work builds on it. A quite similar approach was proposed in \cite{mahajan2015efficient}, with richer class class of subproblems that can be formulated locally, and solved approximately. An analysis of inexact version of DANE  and its accelerated variant, AIDE, appeared recently in \cite{AIDE}. Inexact DANE is closely related to the algorithms presented in this paper. We, however, continue in different direction shaped by the setting of \fedopt.

The DiSCO algorithm \cite{DiSCO} of Zhang and Xiao is based on inexact damped Newton method. The core idea is that the inexact Newton steps are computed by distributed preconditioned conjugate gradient, which can be very fast, if the data are distributed in an \iid fashion, enabling a good preconditioner to be computed locally. The theoretical upper bound on number of rounds of communication improves upon DANE and other methods, and in certain settings matches the lower bound presented in \cite{CommShamir}. The DiSCO algorithm is related to \cite{lin2014large, zhuang2015distributed}, a distributed truncated Newton method. Although it was reported to perform well in practice, the total number of conjugate gradient iterations may still be high to be considered a communication efficient algorithm.

Common to the above algorithms is the assumption that each \node has access to data points sampled \iid from the same distribution. This assumption is not required only in theory, but can cause the algorithms to converge significantly slower or even diverge (as reported for instance in \cite[Table 3]{DANE}). Thus, these algorithms, at least in their default form, are not suitable for the setting of Federated Optimization presented here.

An algorithm that bypasses the need for \iid data assumption is CoCoA, which provably converges under any distribution of the data, while the convergence rate does depend on properties of the data distribution. The first version of the algorithm was proposed as DisDCA in \cite{yang2013trading}, without convergence guarantees. First analysis was introduced in \cite{cocoaNIPS}, with further improvements in \cite{cocoaICML}, and a more general version in \cite{ma2015distributed}. Recently, its variant for L1-regularized objectives was introduced in \cite{smith2015l1}.

The CoCoA framework formulates general local subproblems based on the dual form of \eqref{eq:problem} (See for instance \cite[Eq.\ (2)]{ma2015distributed}). Data points are distributed to \nodes, along with corresponding dual variables. Arbitrary optimization algorithm is used to attain a relative $\Theta$ accuracy on the local subproblem --- by changing only local dual variables. These updates have their corresponding updates to primal variable $w$, which are synchronously aggregated (could be averaging, adding up, or anything in between; depending on the local subproblem formulation).

From the description in this section it appears that the CoCoA framework is the only usable tool for the setting of Federated Optimization. However, the theoretical bound on number of rounds of communications for ill-conditioned problems scales with the number of \nodes $K$. Indeed, as we will show in Section~\ref{sec:experiments} on real data, CoCoA framework does converge very slowly.

\section{Algorithms for \FedOpt}
\label{sec:algorithms}

In this section we introduce the first algorithm that was designed with the unique challenges of \fedopt in mind. Before proceeding with the explanation, we first revisit two important and at first sight unrelated algorithms. The connection between these algorithms helped to motivate our research. Namely, the algorithms are the Stochastic Variance Reduced Gradient (\SVRG) \cite{SVRG, S2GD}, a stochastic method with explicit variance reduction, and the Distributed Approximate Newton (DANE) \cite{DANE} for distributed optimization.

The descriptions are followed by their connection, giving rise to a new distributed optimization algorithm, at first sight almost identical to the \SVRG algorithm, which we call Federated \SVRG (\algname).

Although this algorithm seems to work well in practice in simple circumstances, its performance is still unsatisfactory in the general setting we specify in Section~\ref{sec:problem}. We proceed by making the FSVRG algorithm adaptive to different local data sizes, general sparsity patterns and significant differences in patterns in data available locally, and those present in the entire data set.

\subsection{Desirable Algorithmic Properties}
It is a useful thought experiment to consider the properties one would
hope to find in an algorithm for the non-\iid, unbalanced, and
massively-distributed setting we consider.  In particular:
\begin{enumerate}[(A)]
\item \label{propStartOpt} If the algorithm is initialized to the optimal solution, it stays there.
\item \label{propOneNode} If all the data is on a single \node, the algorithm should converge in $\BO(1)$ rounds of communication.
\item \label{propDisjointFeatures} If each feature occurs on a single \node, so the problems are fully decomposable (each machine is essentially learning a disjoint block of parameters), then the algorithm should converge in $\BO(1)$ rounds of communication\footnote{This is valid only for generalized linear models.}.
\item \label{propIdentical} If each \node contains an identical dataset, then the algorithm should converge in $\BO(1)$ rounds of communication. 
\end{enumerate}
For convex problems, ``converges'' has the usual technical meaning of finding a solution sufficiently close to the global minimum, but these properties also make sense for non-convex problems where ``converge'' can be read as ``finds a solution of sufficient quality''. In these statements, $\BO(1)$ round is ideally exactly one round of communication.

Property \ref{propStartOpt} is valuable in any optimization setting.
Properties \ref{propOneNode} and \ref{propDisjointFeatures} are
extreme cases of the \fedopt setting (non-\iid, unbalanced, and sparse),
whereas \ref{propIdentical} is an extreme case of the classic
distributed optimization setting (large amounts of \iid data per
machine). Thus, \ref{propIdentical} is the least important property
for algorithms in the \fedopt setting.

\subsection{\SVRG}
\label{sec:SVRG}

The \SVRG algorithm \cite{SVRG, S2GD} is a stochastic method designed to solve problem~\eqref{eq:problem} on a single \node. We present it as Algorithm~\ref{alg:S2GD} in a slightly simplified form.

\begin{algorithm}[!h]
\begin{algorithmic}[1]
\State \textbf{parameters:} $m$ = number of stochastic steps per epoch, $h$ = stepsize
\For {$s = 0, 1, 2, \dots$}
	\State Compute and store $\grad f(w^t) = \frac{1}{n} \sum_{i=1}^n \grad f_i(w^t)$ \label{line:fullgrad}
	\Comment Full pass through data
	\State Set $w = w^t$
	\For {$t = 1$ to $m$}
		\State Pick $i \in \{ 1, 2, \dots, n \}$, uniformly at random
		\State $w = w - h \left( \grad f_i(w) - \grad f_i(w^t) + \grad f(w^t) \right) $ 
		\Comment Stochastic update \label{line:stoch_update}
	\EndFor
	\State $w^{t+1} = w$
\EndFor
\end{algorithmic}

\caption{\SVRG}
\label{alg:S2GD}
\end{algorithm}

The algorithm runs in two nested loops. In the outer loop, it computes gradient of the entire function $f$ (Line~\ref{line:fullgrad}). This constitutes for a full pass through data --- in general expensive operation one tries to avoid unless necessary. This is followed by an inner loop, where $m$ fast stochastic updates are performed. In practice, $m$ is typically set to be a small multiple (1--5) of $n$. Although the theoretically optimal choice for $m$ is a small multiple of a condition number \cite[Theorem 6]{S2GD}, this is often of the same order as $n$ in practice.

The central idea of the algorithm is to avoid using the stochastic gradients to estimate the entire gradient $\grad f(w)$ directly. Instead, in the stochastic update in Line \ref{line:stoch_update}, the algorithm evaluates two stochastic gradients, $\grad f_i(w)$ and $\grad f_i(w^t)$. These gradients are used to estimate the change of the gradient of the entire function between points $w^t$ and $w$, namely $\grad f(w) - \grad f(w^t)$. Using this estimate together with $\grad f(w^t)$ pre-computed in the outer loop, yields an unbiased estimate of $\grad f(w)$.

Apart from being an unbiased estimate, it could be intuitively clear that if $w$ and $w^t$ are close to each other, the variance of the estimate $\grad f_i(w) - \grad f_i(w^t)$ should be small, resulting in estimate of $\grad f(w)$ with small variance. As the inner iterate $w$ goes further, variance grows, and the algorithm starts a new outer loop to compute new full gradient $\grad f(w^{t+1})$ and reset the variance.

The performance is well understood in theory. For $\lambda$-strongly convex $f$ and $L$-smooth functions $f_i$, convergence results are in the form
\begin{equation}
\label{eq:S2GD:convergence}
\Exp [f(w^t) - f(w^*)] \leq c^t [f(w^0) - f(w^*)],
\end{equation}
where $w^*$ is the optimal solution, and $c = \Theta \left( \frac{1}{m h} \right) + \Theta(h)$.\footnote{See \cite[Theorem 4]{S2GD} and \cite[Theorem 1]{SVRG} for details.}

It is possible to show \cite[Theorem 6]{S2GD} that for appropriate choice of parameters $m$ and $h$, the convergence rate~\eqref{eq:S2GD:convergence} translates to the need of $$ \left( n + \BO ( L/\lambda) \right) \log(1/\epsilon) $$
evaluations of $\grad f_i$ for some $i$ to achieve $\Exp [f(w) - f(w^*)] < \epsilon$.

\subsection{Distributed Problem Formulation}
\label{sec:problem}

In this section, we introduce notation and specify the structure of the distributed version of the problem we consider \eqref{eq:problem}, focusing on the case where the $f_i$ are convex.
We assume the data $\{x_i, y_i\}_{i=1}^n$, describing functions $f_i$ are stored
across a large number of \nodes.

Let $K$ be the number of \nodes.  Let $\pp_k$ for $k \in \{1, \dots, K\}$ denote a partition of data point indices $\{1, \dots, n\}$, so $\pp_k$ is the set stored on \node $k$, and define $n_k = |\pp_k|$. That is, we assume that $\pp_k \cap \pp_l = \emptyset$ whenever $k \neq l$, and $\sum_{k=1}^K n_k = n$.
We then define local empirical loss as
\begin{equation}\label{eq:98hs98hs8}
F_k(w) \eqdef \frac{1}{n_k} \sum_{i \in \mathcal{P}_k} f_i(w),
\end{equation}
which is the local objective based on the data stored on machine $k$. We can then rephrase the objective~\eqref{eq:problem} as
\begin{equation}
\label{eq:problem:distributed}
f(w) = \sum_{k = 1}^K \frac{n_k}{n} F_k(w) 
= \sum_{k=1}^K \frac{n_k}{n} \cdot \frac{1}{n_k} \sum_{i \in \mathcal{P}_k} f_i(w).
\end{equation}

The way to interpret this structure is to see the empirical loss $f(w) = \frac1n \sum_{i=1}^n f_i(w)$ as a convex combination of the local empirical losses $F_k(w)$, available locally to \node $k$. Problem \eqref{eq:problem} then takes the simplified form
\begin{equation}
\label{eq:problem:distributed:simple}
\min_{w\in \R^d} f(w) \equiv \sum_{k=1}^K \frac{n_k}{n} F_k(w).
\end{equation}

\subsection{DANE}
\label{sec:algorithms:DANE}

In this section, we introduce a general reasoning providing stronger intuitive support for the DANE algorithm \cite{DANE}, which we describe in detail below. We will follow up on this reasoning in Appendix~\ref{sec:appendix} and draw a connection between two existing methods that was not known in the literature.

If we wanted to design a distributed algorithm for solving the above problem \eqref{eq:problem:distributed:simple}, where \node $k$ contains the data describing function $F_k$. The first, and as we shall see, a rather naive idea is to ask each node to minimize their local functions, and average the results (a variant of this idea appeared in \cite{zinkevich2010parallelized}):
$$ w_k^{t+1} = \arg \min_{w\in \R^d} F_k(w), 
\qquad w^{t+1} = \sum_{k=1}^K \frac{n_k}{n} w_k^{t+1}. $$

Clearly, it does not make sense to run this algorithm for more than one iteration as the output $w$ will always be the same. This is simply because $w_k^{t+1}$ does not depend on $t$. In other words, this method effectively performs just a single round of communication. While the simplicity is appealing, the drawback of this method is that it can't work. Indeed, there is no reason to expect that in general the solution of \eqref{eq:problem:distributed:simple} will be a weighted average of the local solutions, unless the local functions are all the same --- in which case we do not need a distributed algorithm in the first place and can instead solve the much simpler problem $\min_{w\in \R^d} F_1(w)$. This intuitive reasoning can be also formally supported, see for instance \cite[Appendix A]{DANE}.

One remedy to the above issue is to modify the local problems before each aggregation step. One of the simplest strategies would be to perturb the local function $F_k$ in iteration $t$ by a quadratic term of the form: $-( a_k^t)^T w + \tfrac{\mu}{2}\|w-w^t\|^2$ and to ask each node to solve the perturbed problem instead. With this change, the improved method then takes the form
\begin{equation}
\label{eq:alg:perturbation}
w_k^{t+1} = \arg \min_{w\in \R^d} F_k(w) - (a_k^t)^T w  + \frac{\mu}{2}\|w-w^t\|^2, \qquad w^{t+1} = \frac{1}{K}\sum_{k=1}^K w_k^{t+1}.
\end{equation}

The idea behind iterations of this form is the following. We would like each node $k\in [K]$ to use as much curvature information stored in $F_k$ as possible. By keeping the function $F_k$ in the subproblem in its entirety, we are keeping the curvature information nearly intact --- the Hessian of the subproblem is $\nabla^2 F_k + \mu I$, and we can even choose $\mu = 0$.

As described, the method is not yet well defined, since we have not described how the vectors $a_k^t$ would change from iteration to iteration, and how one should choose $\mu$.  In order to get some insight into how such a method might work, let us examine the optimality conditions. Asymptotically as $t\to \infty$, we would like $a_k^t$ to be such that the minimum of each subproblem is equal to $w^*$; the minimizer of \eqref{eq:problem:distributed:simple}. Hence, we would wish for $w^*$ to be the solution of 
$$ \nabla F_k(w) - a_k^t +\mu(w - w^t) = 0. $$

Hence, in the limit, we would ideally like to choose $a_k^t = \nabla F_k(w^*) + \mu(w^* - w^t) \approx \nabla F_k(w^*)$, since $w^* \approx w^t$.  Not knowing $w^*$ however, we cannot hope to be able to simply set $a_k^t$ to this value. Hence, the second option is to come up with an update rule which would guarantee that $a_k^t $ converges to $\nabla F_k(w^*)$ as $t \to \infty$. Notice at this point that it has been long known in the optimization community that the gradient of the objective at the optimal point is intimately related to the optimal solution of a dual problem. Here the situation is further complicated by the fact that we need to learn $K$ such gradients. In the following, we show that DANE is in fact a particular instantiation of the scheme above.

\paragraph{DANE.} We present the Distributed Approximate Newton algorithm (DANE) \cite{DANE}, as Algorithm~\ref{alg:DANE}. The algorithm was originally analysed for solving the problem of structure \eqref{eq:problem:distributed}, with $n_k$ being identical for each $k$ --- i.e., each computer has the same number of data points. Nothing prevents us from running it in our more general setting though.

\begin{algorithm}[!h]
\caption{Distributed Approximate Newton (DANE)}\label{alg:DANE}
\begin{algorithmic}[1]
\State {\bf Input:} regularizer $\mu \geq 0$, parameter $\eta$ (default: $\mu = 0, \eta = 1$)
\For {$s = 0, 1, 2, \dots $}
  \State Compute $\grad f(w^t) = \frac{1}{n} \sum_{i=1}^n \grad f_i(w^t)$ and distribute to all machines \label{line:gradient}
  \State For each \node $k \in \{1, \dots, K\}$, solve \label{line:subproblem}
  \begin{equation}
  \label{eq:DANEsubproblem}
  w_k = \argmin_{w \in \R^d} \left\{ F_k(w) - \left( \grad F_k(w^t) - \eta \grad f(w^t)  \right)^T w  + \frac{\mu}{2} \| w - w^t \|^2 \right\}
  \end{equation}
  \State Compute $w^{t+1} = \frac{1}{K}\sum_{k=1}^K w_k$ \label{line:aggregate}
\EndFor
\end{algorithmic}
\end{algorithm}

As alluded to earlier, the main idea of DANE is to form a local subproblem, dependent only on local data, and gradient of the entire function --- which can be computed in a single round of communication (Line \ref{line:gradient}). The subproblem is then solved exactly (Line \ref{line:subproblem}), and updates from individual \nodes are averaged to form a new iterate (Line \ref{line:aggregate}). This approach allows any algorithm to be used to solve the local subproblem~\eqref{eq:DANEsubproblem}. As a result, it often achieves communication efficiency in the sense of requiring expensive local computation between rounds of communication, hopefully rendering the time needed for communication insignificant (see Section~\ref{sec:relatedWork:paradigm}). Further, note that DANE belongs to the family of distributed method that operate via the quadratic perturbation trick \eqref{eq:alg:perturbation} with 
$$ a_k^t = \nabla F_k(w^t) - \eta \nabla f(w^t). $$
If we assumed that the method works, i.e., that $w^t \to w^*$ and hence $\nabla f(w^t) \to \nabla f(w^*) = 0$, then $a_k^t \to \nabla F_k(w^*)$, which agrees with the earlier discussion.

In the default setting when $\mu = 0$ and $\eta = 1$, DANE achieves desirable property
\ref{propIdentical} (immediate convergence when all local datasets are
identical), since in this case $\grad F_k(w^t) - \eta \grad f(w^t)
= 0$, and so we exactly minimize $F_k(w) = f(w)$ on each machine.
For any choice of $\mu$ and $\eta$, DANE also achieves property
\ref{propStartOpt}, since in this case $\grad f(w^t) = 0$, and
$w^t$ is a minimizer of $F_k(w) - \grad F_k(w^t)\cdot w$ as well
as of the regularization term.
Unfortunately, DANE does not achieve the more \fedopt-specific desirable properties \ref{propOneNode} and \ref{propDisjointFeatures}.

The convergence analysis for DANE assumes that the functions are twice differentiable, and relies on the assumption that each \node has access to \iid samples from the same underlying distribution. This implies that that the Hessians of $\grad^2 F_k(w)$ are similar to each other \cite[Lemma 1]{DANE}. In case of linear regression, with $\lambda = \BO(1 / \sqrt{n})$-strongly convex functions, the number of DANE iterations needed to achieve $\epsilon$-accuracy is $\BO(K \log(1/\epsilon))$. However, for general $L$-smooth loss, the theory is significantly worse, and does not match its practical performance.

The practical performance also depends on the additional local regularization parameter $\mu$. For small number of \nodes $K$, the algorithm converges quickly with $\mu = 0$. However, as reported \cite[Figure 3]{DANE}, it can diverge quickly with growing $K$. Bigger $\mu$ makes the algorithm more stable at the cost of slower convergence. Practical choice of $\mu$ remains an open question.

\subsection{\SVRG meets DANE}

As we mentioned above, the DANE algorithm can perform poorly in certain settings, even without the challenging aspects of \fedopt. Another point that is seen as drawback of DANE is the need to find the \emph{exact} minimum of~\eqref{eq:DANEsubproblem} --- this can be feasible for quadratics with relatively small dimension, but infeasible or extremely expensive to achieve for other problems. We adapt the idea from the CoCoA algorithm \cite{ma2015distributed}, in which an arbitrary optimization algorithm is used to obtain relative $\Theta$ accuracy on a locally defined subproblem. We replace the exact optimization with an approximate solution obtained by using any optimization algorithm.

Considering all the algorithms one could use to solve~\eqref{eq:DANEsubproblem}, the \SVRG algorithm seems to be a particularly good candidate. Starting the local optimization of \eqref{eq:DANEsubproblem} from point $w^t$, the algorithm automatically has access to the derivative at $w^t$, which is identical for each \node\xspace --- $\grad f(w^t)$. Hence, the \SVRG algorithm can skip the initial expensive operation, evaluation of the entire gradient (Line~3, Algorithm~\ref{alg:S2GD}), and proceed only with the stochastic updates in the inner loop.

It turns out that this modified version of the DANE algorithm is equivalent to a distributed version of \SVRG.

\begin{proposition}
\label{prop:equivalence}
Consider the following two algorithms.
\begin{enumerate}
\item Run the DANE algorithm (Algorithm~\ref{alg:DANE}) with $\eta = 1$ and $\mu = 0$, and use \SVRG (Algorithm~\ref{alg:S2GD}) as a local solver for \eqref{eq:DANEsubproblem}, running it for a single iteration, initialized at point $w^t$.
\item Run a distributed variant of the \SVRG algorithm, described in Algorithm~\ref{alg:DS2GDnaive}.
\end{enumerate}

The algorithms are equivalent in the following sense. If both start from the same point $w^t$, they generate identical sequence of iterates $\{ w^t \}$.
\end{proposition}

\begin{proof}
We construct the proof by showing that single step of the \SVRG algorithm applied to the problem \eqref{eq:DANEsubproblem} on computer $k$ is identical to the update on Line~\ref{line:DSVRGupdate} in Algorithm~\ref{alg:DS2GDnaive}.

The way to obtain a stochastic gradient of \eqref{eq:DANEsubproblem} is to sample one of the functions composing $F_k(w) = \frac{1}{n_k} \sum_{i \in \mathcal{P}_k} f_i(w)$, and add the linear term $\grad F_k(w^t) - \eta f(w^t)$, which is known and does not need to be estimated. Upon sampling an index $i \in \mathcal{P}_k$, the update direction follows as 
$$ \left[ \grad f_i(w) - \grad F_k(w^t) - f(w^t) \right] - \left[ \grad f_i(w^t) - \grad F_k(w^t) - f(w^t) \right] + \grad f(w^t) = \grad f_i(w) - \grad f_i(w^t) + \grad f(w^t), $$
which is identical to the direction in Line~\ref{line:DSVRGupdate} in Algorithm~\ref{alg:DS2GDnaive}. The claim follows by chaining the identical updates to form identical iterate $w^{t+1}$.
\end{proof}

\begin{algorithm}[!h]
\begin{algorithmic}[1]
\State \textbf{parameters:} $m$ = \# of stochastic steps per epoch, $h$ = stepsize, data partition $\{\pp_k\}_{k=1}^K$
\For {$s = 0, 1, 2, \dots$}
	\Comment Overall iterations
	\State Compute $\grad f(w^t) = \frac{1}{n} \sum_{i=1}^n \grad f_i(w^t)$
	\For {$k = 1$ to $K$} \textbf{in parallel} over \nodes $k$
	\Comment Distributed loop
	\State Initialize: $w_k = w^t$
	\For {$t = 1$ to $m$}
		\Comment Actual update loop
		\State Sample $i \in \mathcal{P}_k$ uniformly at random
		\State $ w_k = w_k - h \left( \grad f_i(w_k) - \grad f_i(w^t) + \grad f(w^t) \right) $
		\label{line:DSVRGupdate}
		\EndFor
	\EndFor
	\State $w^{t+1} = w^t + \frac{1}{K} \sum_{k=1}^K (w_k - w^t)$
	\Comment Aggregate
\EndFor
\end{algorithmic}

\caption{naive Federated \SVRG (\algname)}
\label{alg:DS2GDnaive}
\end{algorithm}

\begin{remark}
The algorithms considered in Proposition~\ref{prop:equivalence} are inherently stochastic. The statement of the proposition is valid under the assumption that in both cases, identical sequence of samples $i \in \mathcal{P}_k$ would be generated by all \nodes $k \in \{1, 2, \dots, K\}$.
\end{remark}

\begin{remark}
In the Proposition~\ref{prop:equivalence} we consider the DANE algorithm with particular values of $\eta$ and $\mu$. The Algorithm~\ref{alg:DS2GDnaive} and the Proposition can be easily gereralized, but we present only the default version for the sake of clarity.
\end{remark}

Since the first version of this paper, this connection has been mentioned in \cite{AIDE}, which analyses an inexact version of the DANE algorithm. We proceed by adapting the above algorithm to other challenges arising in the context of \fedopt.

\subsection{Federated \SVRG}

Empirically, the Algorithm~\ref{alg:DS2GDnaive} fits in the model of distributed optimization efficiency described in Section~\ref{sec:relatedWork:paradigm}, since we can balance how many stochastic iterations should be performed locally against communication costs. However, several modifications are necessary to achieve good performance in the full \fedopt setting (Section~\ref{sec:problem}). 
Very important aspect that needs to be addressed is that the number of data points available to a given node can differ greatly from the average number of data points available to any single \node. Furthermore, this setting always comes with the data available locally being clustered around a specific pattern, and thus not being a representative sample of the overall distribution we are trying to learn.
In the Experiments section we focus on the case of L2 regularized logistic regression, but the ideas carry over to other generalized linear prediction problems.

\subsubsection{Notation}

Note that in large scale generalized linear prediction problems, the data arising are almost always sparse, for example due to bag-of-words style feature representations. This means that only a small subset of $d$ elements of vector $x_i$ have nonzero values. In this class of problems, the gradient $\grad f_i(w)$ is a multiple of the data vector $x_i$. This creates additional complications, but also potential for exploitation of the problem structure and thus faster algorithms. Before continuing, let us summarize and denote a number of quantities needed to describe the algorithm. 
\begin{itemize}[noitemsep]
\item $n$ --- number of data points / training examples / functions.
\item $\mathcal{P}_k$ --- set of indices, corresponding to data points stored on device $k$.
\item $n_k = |\mathcal{P}_k|$ --- number of data points stored on device $k$.
\item $n^j = \left| \{ i \in \{ 1, \dots, n\} : x_i^T e_j \neq 0 \} \right|$ --- the number of data points with nonzero $j^{th}$ coordinate
\item $n_k^j = \left| \{  i \in \mathcal{P}_k : x_i^T e_j \neq 0 \} \right| $ --- the number of data points stored on \node $k$ with nonzero $j^{th}$ coordinate
\item $\phi^j = n^j / n$ --- frequency of appearance of nonzero elements in $j^{th}$ coordinate
\item $\phi_k^j = n_k^j / n_k$ --- frequency of appearance of nonzero elements in $j^{th}$ coordinate on \node $k$
\item $s_k^j = \phi^j / \phi_k^j$ --- ratio of global and local appearance frequencies on \node $k$ in $j^{th}$ coordinate
\item $S_k = \text{Diag}(s_k^j)$ --- diagonal matrix, composed of $s_k^j$ as $j^{th}$ diagonal element
\item $\omega^j = \left|\{ \mathcal{P}_k : n_k^j \neq 0 \}\right|$ --- Number of \nodes that contain data point with nonzero $j^{th}$ coordinate
\item $a^j = K / \omega^j$ --- aggregation parameter for coordinate $j$
\item $A = \text{Diag}(a_j)$ --- diagonal matrix composed of $a_j$ as $j^{th}$ diagonal element
\end{itemize}

With these quantities defined, we can state our proposed algorithm as Algorithm~\ref{alg:DS2GDv7}. Our experiments show that this algorithm works very well in practice, but the motivation for the particular scaling of the updates may not be immediately clear.  In the following section we provide the intuition that lead to the development of this algorithm.

\begin{algorithm}[!h]
\begin{algorithmic}[1]
\State \textbf{parameters:} $h$ = stepsize, data partition $\{\pp_k\}_{k=1}^K$, \newline {\color{white}} \hspace{64pt} diagonal matrices $A, S_k \in \R^{d \times d}$ for $k \in \{1, \dots, K\}$ 
\For {$s = 0, 1, 2, \dots$}
	\Comment Overall iterations
	\State Compute $\grad f(w^t) = \frac{1}{n} \sum_{i=1}^n \grad f_i(w^t)$
	\For {$k = 1$ to $K$} \textbf{in parallel} over \nodes $k$
	\Comment Distributed loop
	\State Initialize: $w_k = w^t$ and $h_k = h / n_k$
	\State Let $\{ i_t \}_{t=1}^{n_k}$ be random permutation of $\mathcal{P}_k$
	\For {$t = 1, \dots, n_k$}
		\Comment Actual update loop
		\State $ w_k = w_k - h_k \left( S_k \left[ \grad f_{i_t}(w_k) - \grad f_{i_t}(w^t) \right] + \grad f(w^t) \right) $
		\EndFor
	\EndFor
	\State $w^t = w^t + A \sum_{k=1}^K \frac{n_k}{n} (w_k - w^t)$
	\Comment Aggregate
\EndFor
\end{algorithmic}
\caption{Federated \SVRG (\algname)}
\label{alg:DS2GDv7}
\end{algorithm}

\subsubsection{Intuition Behind \algname Updates}
\label{sec:algorithms:intuition}

The difference between the Algorithm~\ref{alg:DS2GDv7} and Algorithm~\ref{alg:DS2GDnaive} is in the introduction of the following properties.

\begin{enumerate}
\item Local stepsize --- $h_k = h / n_k$.
\item Aggregation of updates proportional to partition sizes --- $\frac{n_k}{n} (w_k - w^t)$
\item Scaling stochastic gradients by diagonal matrix --- $S_k$
\item Per-coordinate scaling of aggregated updates --- $A (w_k - w^t)$
\end{enumerate}

Let us now explain what motivated us to get this particular implementation. 

As a simplification, assume that at some point in time, we have for some $w$, $w_k = w$ for all $k \in [K]$. In other words, all the \nodes have the same local iterate. Although this is not exactly the case in practice, thinking about the issue in this simplified setting will give us insight into what would be meaningful to do if it was true. Further, we can hope that the reality is not too far from the simplification and it will still work in practice. Indeed, all \nodes do start from the same point, and adding the linear term $\grad F_k(w^t) - \grad f(w^t)$ to the local objective forces all \nodes to move in the same direction, at least initially.

Suppose the \nodes are about to make a single step synchronously. Denote the update direction on \node $k$ as $G_k = \grad f_i(w) - \grad f_i(w^t) + \grad f(w^t)$, where $i$ is sampled uniformly at random from $\mathcal{P}_k$.

If we had only one \node, i.e., $K = 1$, it is clear that we would have $\Exp[G_1] = \grad f(w^t)$. If $K$ is more than $1$, the values of $G_k$ are in general biased estimates of $\grad f(w^t)$. We would like to achieve the following: $\Exp \left[\sum_{k=1}^K \alpha_k G_k \right] = \grad f(w^t)$, for some choice of $\alpha_k$. This is motivated by the general desire to make stochastic first-order methods to make a gradient step in expectation.

We have 
\begin{equation*}
\Exp \left[\sum_{k=1}^K \alpha_k G_k \right] = \sum_{k=1}^K \alpha_k \frac{1}{n_k} \sum_{i \in \mathcal{P}_k} \left[ \grad f_i(w) - \grad f_i(w^t) + \grad f(w^t) \right].
\end{equation*}
By setting $\alpha_k = \frac{n_k}{n}$, we get
\begin{equation*}
\Exp \left[\sum_{k=1}^K \alpha_k G_k \right] = \frac{1}{n} \sum_{k=1}^K \sum_{i \in \mathcal{P}_k} \left[ \grad f_i(w) - \grad f_i(w^t) + \grad f(w^t) \right] = \grad f(w).
\end{equation*}

This motivates the aggregation of updates from \nodes proportional to $n_k$, the number of data points available locally (Point 2).

Next, we realize that if the local data sizes, $n_k$, are not identical, we likely don't want to do the same number of local iterations on each \node $k$. Intuitively, doing one pass through data (or a fixed number of passes) makes sense. As a result, the aggregation motivated above does not make perfect sense anymore. Nevertheless, we can even it out, by setting the stepsize $h_k$ inversely proportional to $n_k$, making sure each \node makes progress of roughly the same magnitude overall. Hence, $h_k = h / n_k$ (Point 1).

To motivate the Point 3, scaling of stochastic gradients by diagonal matrix $S_k$, consider the following example. We have $1,000,000$ data points, distributed across $K = 1,000$ \nodes. When we look at a particular feature of the data points, we observe it is non-zero only in $1,000$ of them. Moreover, all of them happen to be stored on a single \node, that stores only these $1,000$ data points. Sampling a data point from this \node and evaluating the corresponding gradient, will clearly yield an estimate of the gradient $\grad f(w)$ with $1000$-times larger magnitude. This would not necessarily be a problem if done only once. However, repeatedly sampling and overshooting the magnitude of the gradient will likely cause the iterative process to diverge quickly.

Hence, we scale the stochastic gradients by a diagonal matrix. This can be seen as an attempt to enforce the estimates of the gradient to be of the correct magnitude, conditioned on us, algorithm designers, being aware of the structure of distribution of the sparsity pattern.

Let us now highlight some properties of the modification in Point 4. Without any extra information, or in the case of fully dense data, averaging the local updates is the only way that actually makes sense --- because each \node outputs approximate solution of a proxy to the overall objective, and there is no induced separability structure in the outputs such as in CoCoA \cite{ma2015distributed}. However, we could do much more in the other extreme. If the sparsity structure is such that each data point only depends on one of disjoint groups of variables, and the data were distributed according to this structure, we would efficiently have several disjoint problems. Solving each of them locally, and adding up the results would solve the problem in single iteration --- desired algorithm property (C).

What we propose is an interpolation between these two settings, on a per-variable basis. If a variable appears in data on each \node, we are going to take average. However, the less \nodes a particular variable appear on, the more we want to trust those few \nodes in informing us about the meaningful update to this variable --- or alternatively, take a longer step. Hence the per-variable scaling of aggregated updates.

\subsection{Further Notes}
\label{sec:algorithms:furthernotes}

Looking at the Proposition~\ref{prop:equivalence}, we identify equivalence of two algorithms, take the second one and try modify it to make it suitable for the setting of \fedopt. A question naturally arise: Is it possible to achieve the same by modifying the first algorithm suitable for \fedopt\ --- by only altering the local optimization objective? 

We indeed tried to experiment with idea, but we don't report the details for two reasons. First, the requirement of exact solution of the local subproblem is often impractical. Relaxing it gradually moves us to the setting we presented in the previous sections. But more importantly, using this approach we have only managed to get results significantly inferior to those reported later in the Experiments section.

\section{Experiments}
\label{sec:experiments}

In this section we present the first experimental results in the setting of \fedopt. In particular, we provide results on a dataset based on public Google+ posts\footnote{The posts were public at the time the experiment was performed, but since a user may decide to delete the post or make it non-public, we cannot release (or even permanently store) any copies of the data.}, clustered by user --- simulating each user as a independent \node. This preliminary experiment demonstrates why none of the existing algorithms are suitable for \fedopt, and the robustness of our proposed method to challenges arising there. 

\subsection{Predicting Comments on Public Google+ Posts}
\label{sec:experiments:gplus}
The dataset presented here was generated based on public Google+ posts. We randomly picked $10,000$ authors that have at least $100$ public posts in English, and try to predict whether a post will receive at least one comment (that is, a binary classification task). 

We split the data chronologically on a per-author basis, taking the earlier $75\%$ for training and the following $25\%$ for testing. The total number of training examples is $n = 2,166,693$. We created a simple bag-of-words language model, based on the $20,000$ most frequent words in dictionary based on all Google+ data. This results in a problem with dimension $d = 20,002$. The extra two features represent a bias term and variable for unknown word. We then use a logistic regression model to make a prediction based on these features.

We shape the distributed optimization problem as follows. Suppose that each user corresponds to one \node, resulting in $K = 10,000$. The average $n_k$, number of data points on \node $k$ is thus roughly $216$. However, the actual numbers $n_k$ range from $75$ to $9,000$, showing the data is in fact substantially unbalanced.

It is natural to expect that different users can exhibit very different patterns in the data generated. This is indeed the case, and hence the distribution to \nodes cannot be considered an \iid sample from the overall distribution. Since we have a bag-of-words model, our data are very sparse --- most posts contain only small fraction of all the words in the dictionary. This, together with the fact that the data are naturally clustered on a per-user basis, creates additional challenge that is not present in the traditional distributed setting. 

\begin{figure}[!h]
\centering
\includegraphics[width=0.5\textwidth]{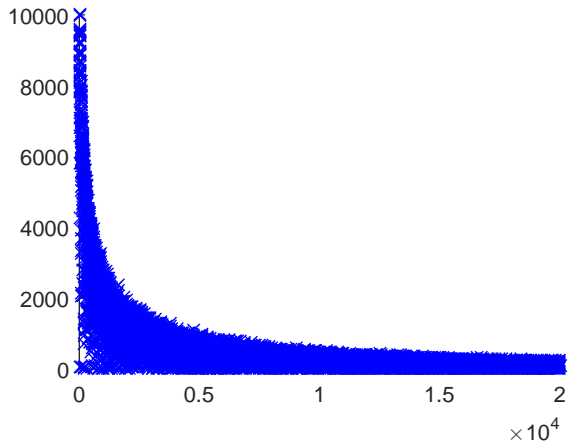}
\caption{Features vs. appearance on \nodes.  The $x$-axis is a feature index, and the $y$-axis represents the number of \nodes where a given feature is present.}
\label{fig:omegaprime}
\end{figure}

Figure~\ref{fig:omegaprime} shows the frequency of different features across \nodes. Some features are present everywhere, such as the bias term, while most features are relatively rare. In particular, over $88\%$ of features are present on fewer than $1,000$ \nodes. However, this distribution does not necessarily resemble the overall appearance of the features in data examples. For instance, while an unknown word is present in data of almost every user, it is far from being contained in every data point.

\paragraph{Naive prediction properties.} Before presenting the results, it is useful to look at some of the important basic prediction properties of the data. We use L2-regularized logistic regression, with regularization parameter $\lambda = 1/n$. We chose $\lambda$ to be the best in terms of test error in the optimal solution.
\begin{itemize} \itemsep -2pt
\item If one chooses to predict $-1$ (no comment), classification error is $\textbf{33.16}\%$. 
\item The optimal solution of the global logistic regression problem yields $\textbf{26.27}\%$ test set error. 
\item Predicting the per-author majority from the training data yields $\textbf{17.14}\%$ test error. That is, predict $+1$ or $-1$ for all the posts of an author, based on which label was more common in that author's training data. This indicates that knowing the author is actually more useful than knowing what they said, which is perhaps not surprising.
\end{itemize}

In summary, this data is representative for our motivating application in \fedopt. It is possible to improve upon naive baseline using a fixed global model. Further, the per-author majority result suggests it is possible to improve further by adapting the global model to each user individually. Model personalization is common practice in industrial applications, and the techniques used to do this are orthogonal to the challenges of \fedopt. Exploring its performance is a natural next step, but beyond the scope of this work.

While we do not provide experiments for per user personalized models, we remark that this could be a good descriptor of how far from \iid the data is distributed. Indeed, if each \node has access to an \iid sample, any adaptation to local data is merely over-fitting. However, if we can significantly improve upon the global model by per user/\node adaptation, this means that the data available locally exhibit patterns specific to the particular \node.

The performance of the Algorithm~\ref{alg:DS2GDv7} is presented below. The only parameter that remains to be chosen by user is the stepsize $h$. We tried a set of stepsizes, and retrospectively choose one that works best --- a typical practice in machine learning.

\begin{figure}[!h]
\centering
\includegraphics[width=0.48\textwidth]{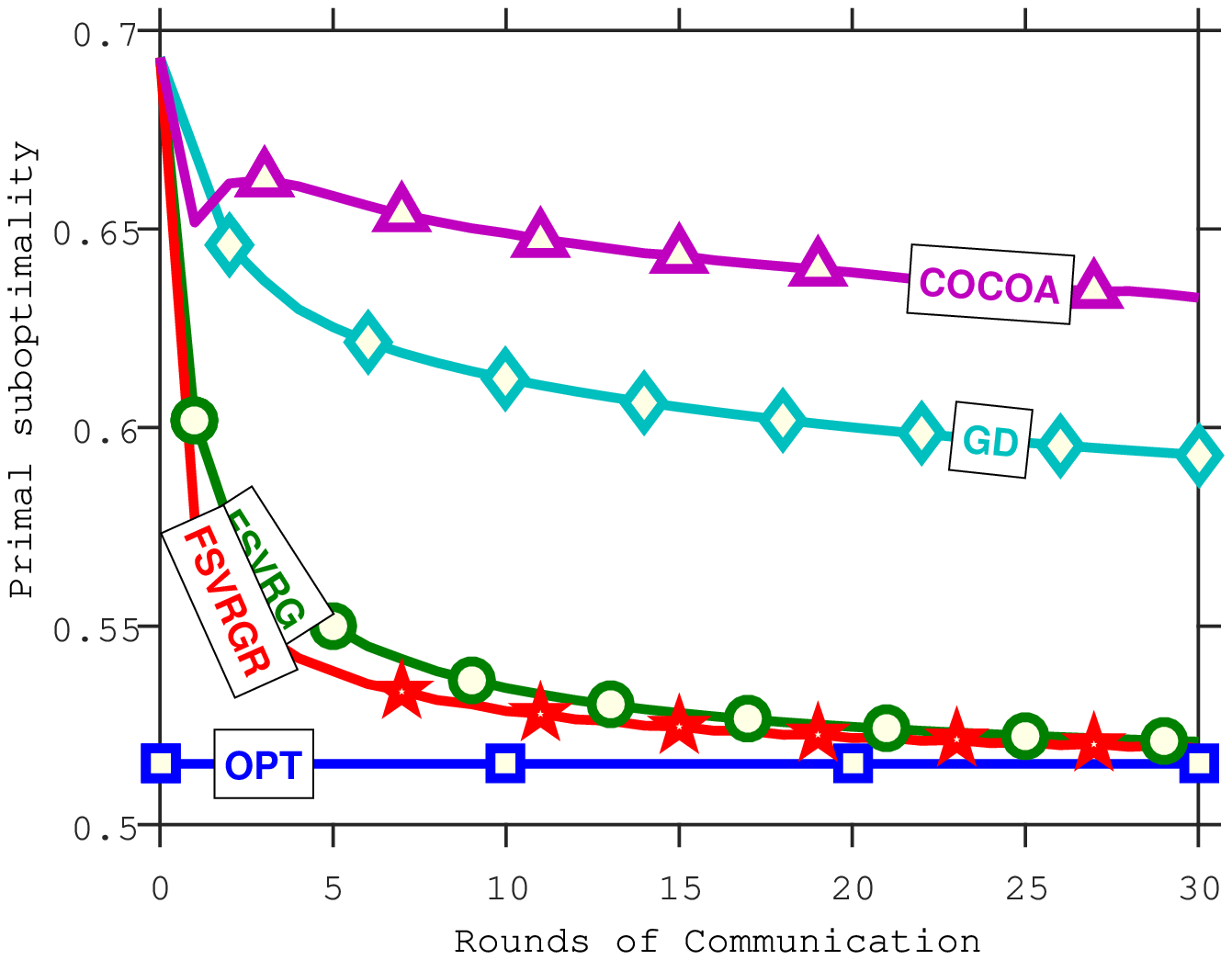}
\includegraphics[width=0.48\textwidth]{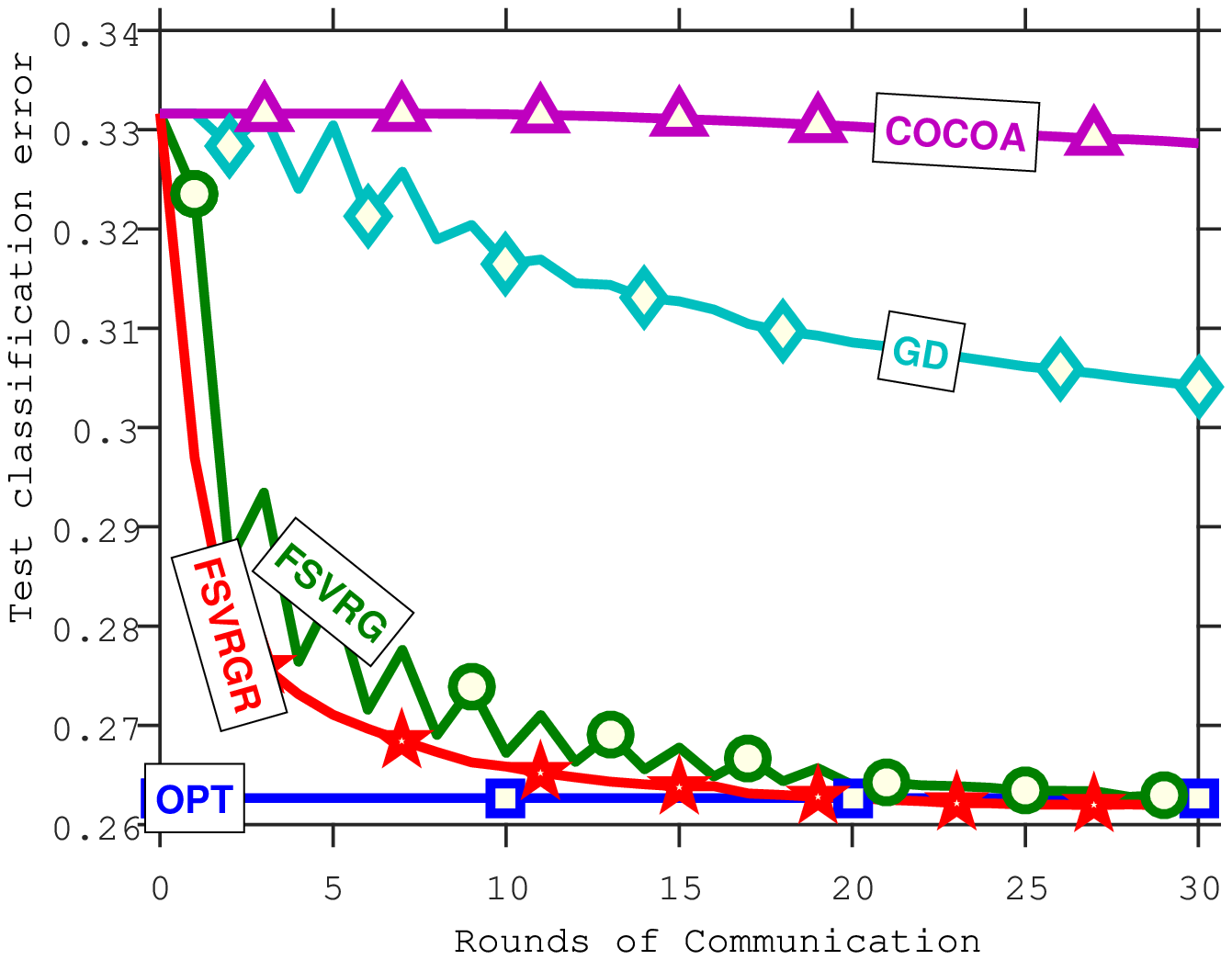}
\caption{Rounds of communication vs. objective function (left) and test prediction error (right).}
\label{fig:ex_final_001}
\end{figure}

In Figure~\ref{fig:ex_final_001}, we compare the following optimization algorithms\footnote{We thank Mark Schmidt for his \texttt{prettyPlot} function, available on his website.}:
\begin{itemize} \itemsep -2pt
 \item The blue squares (OPT) represent the best possible offline value (the optimal value of the optimization task in the first plot, and the test error corresponding to the optimum in the second plot).
 \item  The teal diamonds (GD) correspond to a simple distributed gradient descent. 
 \item The purple triangles (COCOA) are for the CoCoA+ algorithm \cite{ma2015distributed}.
\item The green circles (FSVRG) give values for our proposed algorithm.
\item The red stars (FSVRGR) correspond to the same algorithm applied to the same problem with randomly reshuffled data.  That is, we keep the unbalanced number of examples per node, but populate each node with randomly selected examples.
\end{itemize}

The first thing to notice is that CoCoA+ seems to be worse than trivial benchmark --- distributed gradient descent. This behaviour can be predicted from theory, as the overall convergence rate directly depends on the best choice of aggregation parameter $\sigma'$. For sparse problems, it is upperbounded by the maximum of the values reported in Figure~\ref{fig:omegaprime}, which is $K$, and it is close to it also in practice. Althought it is expected that the algorithm could be modified to depend on average of these quantities (which could be orders of magnitude smaller), akin to coordinate descent algorithms \cite{richtarikCD}, it has not been done yet. Note that other communication efficient algorithms fail to converge altogether.

The algorithm we propose, FSVRG, converges to optimal test classification accuracy in just $30$ iterations. Recall that in the setting of \fedopt we introduced in Section~\ref{sec:intro:challenge}, minimization of rounds of communication is the principal goal. However, concluding that the approach is stunningly superior to existing methods would not be completely fair nor correct. The conclusion is that the \emph{FSVRG is the first algorithm to tackle \fedopt}, a problem that existing methods fail to generalize to. It is important to stress that none of the existing methods were designed with these particular challenges in mind, and we formulate the first benchmark.

Since the core reason other methods fail to converge is the non-IID data distribution, we test our method on the same problem, with data randomly reshuffled among the same number of \nodes (FSVRGR; red stars). Since the difference in convergence is subtle, we can conclude that the techniques described in Section~\ref{sec:algorithms:intuition} serve its purpose and make the algorithm robust to challenges present in \fedopt.

This experiment demonstrates that learning from massively decentralized data, clustered on a per-user basis is indeed problem we can tackle in practice. Since the first version of this paper \cite{konevcny2015federated}, additional experimental results were presented in \cite{mcmahan2016federated}. We refer the reader to this paper for experiments in more challenging setting of deep learning, and a further discussion on how such system would be implemented in practice.

\section{Conclusions and Future Challenges}
\label{sec:conclusions}

We have introduced a new setting for distributed optimization, which we call \emph{\fedopt}. This setting is motivated by the outlined vision, in which users do not send the data they generate to companies at all, but rather provide part of their computational power to be used to solve optimization problems. This comes with a unique set of challenges for distributed optimization. In particular, we argue that the massively distributed, non-\iid, unbalanced, and sparse properties of \fedopt problems need to be addressed by the optimization community.

We explain why existing methods are not applicable or effective in this setting. Even the distributed algorithms that can be applied converge very slowly in the presence of large number of \nodes on which the data are stored. We demonstrate that in practice, it is possible to design algorithms that work surprisingly efficiently in the challenging setting of \fedopt, which makes the vision conceptually feasible.

We realize that it is important to scale stochastic gradients on a per-coordinate basis, differently on each \node to improve performance. To the best of our knowledge, this is the first time such per-\node scaling has been used in distributed optimization. Additionally, we use per-coordinate aggregation of updates from each \node, based on distribution of the sparsity patterns in the data.

Even though our results are encouraging, there is a lot of room for future work.  One natural direction is to consider fully asynchronous versions of our algorithms, where the updates are applied as soon as they arrive. Another is developing a better theoretical understanding of our algorithm, as we believe that development of a strong understanding of the convergence properties will drive further research in this area.

Study of the \fedopt problem for non-convex objectives is another important avenue of research. In particular, neural networks are the most important example of a machine learning tool that yields non-convex functions $f_i$, without any convenient general structure. Consequently, there are no useful results describing convergence guarantees of optimization algorithms. Despite the lack of theoretical understanding, neural networks are now state-of-the-art in many application areas, ranging from natural language understanding to visual object detection.  Such applications arise naturally in \fedopt settings, and so extending our work to such problems is an important direction.

The non-\iid data distribution assumed in \fedopt, and mobile applications in particular, suggest that one should consider the problem of training a \emph{personalized} model together with that of learning a global model.  That is, if there is enough data available on a given node, and we assume that data is drawn from the same distribution as future test examples for that node, it may be preferable to make predictions based on a personalized model that is biased toward good performance on the local data, rather than simply using the global model.

{\footnotesize
\bibliography{paper}
}

\clearpage
\appendix

\section{Distributed Optimization via Quadratic  Perturbations}
\label{sec:appendix}

This appendix follows from the discussion motivating DANE algorithm by a general algorithmic perturbation template~\eqref{eq:alg:perturbation} for $\lambda$-strongly convex objectives. We use this to propose a similar but new method, which unlike DANE converges under arbitrary data partitioning $\{\mathcal{P}_k\}_{k=1}^K$, and we highlight its relation to the dual CoCoA algorithm for distributed optimization.

For simplicity and ease of drawing the above connections we assume that $n_k$ is identical for all $k \in \{ 1, 2, \dots, K \}$ throughout the appendix. All the arguments can be simply extended, but would unnecessarily complicate the notation for current purpose.

\subsection{New Method}

We now present a new method (Algorithm~\ref{alg:PRIMAL}), which also belongs to the family of quadratic perturbation methods \eqref{eq:alg:perturbation}. However, the perturbation vectors $a_k^t$ are different from those of DANE. In particular, we set
\[a_k^t  \eqdef \nabla F_k(w^t) - (\eta \nabla F_k(w^t) +  g_k^t),\]
where $\eta>0$ is a parameter, and the vectors $g_k^t$ are maintained by the method. As we show in Lemma~\ref{lem:sum_zero}, Algorithm~\ref{alg:PRIMAL} satisfies \[\sum_{k=1}^K g_k^t = 0\] for all iterations $t$. This implies that  $\tfrac{1}{K}\sum_{k=1}^K a_k^t = (1-\eta) \nabla f(w^t)$. That is, both DANE and the new method use a linear perturbation which, when averaged over the nodes, involves the gradient of the objective function $f$ at the latest iterate $w^t$. Therefore, the methods have one more property in common beyond both being of the form  \eqref{eq:alg:perturbation}. However, as we shall see in the rest of this section, Algorithm~\ref{alg:PRIMAL} allows an insightful dual interpretation. Moreover, while DANE may not converge for arbitrary problems (even when restricted to ridge regression)---and is only known to converge under the assumption that the data stored on each node are in some precise way similar, Algorithm~\ref{alg:PRIMAL} converges for any ridge regression problem and any data partitioning.

Let us denote by $X_k$ the matrix obtained by stacking the data points $x_i$ as column vectors for all $i \in \mathcal{P}_k$. We have the following Lemma.
\begin{lemma}
\label{lem:sum_zero} 
For all $t \geq 0$ we have $\sum_{k=1}^K g_k^t = 0$.
\end{lemma}
\begin{proof}
The statement holds for $t=0$. Indeed,
\[\sum_{k=1}^K g_k^t = \eta\sum_{k=1}^K \left(\frac{K}{n}X_k \alpha_k^0 - \lambda w^0\right) =  0,\]
where the last step follows from the definition of $w^0$. Assume now that the statement hold for $t$. Then
\[\sum_{k=1}^K g_k^{t+1} = \sum_{k=1}^K \left(g_k^t +\eta \lambda(w_k^{t+1} - w^{t+1}) \right)\\
= \eta \lambda \sum_{k=1}^K (w_k^{t+1} - w^{t+1}).\]
The first equation follows from the way $g_k$ is updated in the algorithm. The second equation follows from the inductive assumption, and the last equation follows from the definition of $w^{t+1}$ in the algorithm.
\end{proof}

\begin{algorithm}[t]
\begin{algorithmic}[1]
\State \textbf{Input:} $\sigma \in [1,K]$
\State \textbf{Choose:} $\alpha_k^0\in \R^{|{\cal P}_k|}$ for $k=1,2,\dots,K$ 
\State \textbf{Set:} $\eta = \frac{K}{\sigma}, \; \mu = \lambda (\eta-1)$
\State \textbf{Set:} $w^0 = \frac{1}{\lambda n} \sum_{k=1}^K X_k \alpha_k^0$
\State \textbf{Set:} $g_k^0 = \eta (\frac{K}{n} X_k \alpha_k^0 - \lambda w^0)$ for $k=1,2,\dots,K$
\For {$t = 0,1,2,\dots$}
	 \For {$k = 1$ \textbf{to} $K$} 
	      \State	   $w^{t+1}_k =\arg \min_{w\in \R^d} F_k(w) -  \left( \nabla F_k(w^t)  - (\eta \nabla F_k(w^t)  +   g_k^t)\right)^T w + \frac{\mu}{2}\|w-w^t\|^2 $ 
	\EndFor	
	\State 	 $w^{t+1} = \frac{1}{K} \sum_{k=1}^K w^{t+1}_k$	 	 
	 \For {$k = 1$ \textbf{to} $K$} 
		  \State  $g_k^{t+1} = g_k^t + \lambda \eta (w_k^{t+1} - w^{t+1})$	 
	 \EndFor  	  
\State \textbf{return} $w^t$
\EndFor
\end{algorithmic}
\caption{Primal Method}
\label{alg:PRIMAL}

\end{algorithm}

\subsection{L2-Regularized Linear Predictors}

In the rest of this section we consider the case of L2-regularized {\em linear predictors}.  That is, we focus on problem \eqref{eq:problem}  with $f_i$ of the form
\[f_i(w) = \phi_i(x_i^T w) + \frac{\lambda}{2}\|w\|^2,\]
where $\lambda>0$ is a regularization parameter. 
This leads to L2 regularized empirical risk minimization (ERM) problem
\begin{equation} \label{eq:ERM-linear-pred} \min_{w\in \R^d} \left\{ f(w)\eqdef \frac{1}{n}\sum_{i=1}^n\phi_i(x_i^T w) + \frac{\lambda}{2}\|w\|^2 \right\}.\end{equation}
We assume that the loss functions $\phi_i:\R  \to \R$ are convex and  $1/\gamma$-smooth for some $\gamma>0$; these are standard assumptions.   As usual, we allow the loss function $\phi_i$ to depend on the label $y_i$. For instance, we may choose the quadratic loss: $\phi_i(t) = \tfrac{1}{2}(t-y_i)^2$ (for which $\gamma=1$).

Let $X = [x_1,\dots,x_n]\in \R^{d\times n}$. As described in Section~\ref{sec:problem}, we assume that the data $(x_i,y_i)_{i=1}^n$ is distributed among $K$ nodes of a computer cluster as follows: node $k=1,2,\dots,K$ contains pairs $(x_i,y_i)$ for $i\in {\cal P}_k$, where ${\cal P}_1,\dots,{\cal P}_K$ forms a partition of the set $[n]=\{1,2,\dots,n\}$. Letting $X = [X_1,\dots,X_K]$, where $X_k\in \R^{d\times |{\cal P}_k|}$ is a submatrix of $A$ corresponding to columns $i\in {\cal P}_k$, and $y_k\in \R^{|{\cal P}_k|}$ is the subvector of $y$ corresponding to entries $i\in {\cal P}_k$. Hence, node $k$ contains the pair $(X_k,y_k)$. With this notation, we can write the problem in the form \eqref{eq:problem:distributed:simple}, where 
\begin{equation}
\label{eq:iuhd9898g9} 
F_k(w) = \frac{K}{n}\sum_{i\in {\cal P}_k} \phi_i(x_i^T w) + \frac{\lambda}{2}\|w\|^2.
\end{equation}

\subsection{A Dual Method: Dual Block Proximal Gradient Ascent}

The dual of \eqref{eq:ERM-linear-pred} is the problem

\begin{equation}
\label{eq:dual_9878979s8}
\max_{\alpha\in \R^n} \left\{ D(\alpha) \eqdef - \frac{1}{2\lambda n^2}\left\| X \alpha\right\|^2 - \frac{1}{n}\sum_{i=1}^n\phi_i^*(-\alpha_i ) \right\},
\end{equation}
where $\phi_i^*$ is the convex conjugate of $\phi_i$. Since we assume that $\phi_i$ is $1/\gamma$ smooth, it follows that $\phi_i^*$ is $\gamma$ strongly convex. Therefore, $D$ is a strongly concave function.

\paragraph{From dual solution to a primal solution.} It is well known that if $\alpha^*$ is the optimal solution of the dual problem \eqref{eq:ERM-linear-pred}, then
$w^* \eqdef \frac{1}{\lambda n} X \alpha^*$
is the optimal solution of the primal problem. Therefore, for any dual algorithm producing a sequence of iterates $\alpha^t$, we can define a corresponding primal algorithm  via the linear mapping  \begin{equation}\label{eq:primal_from_dual}w^t \eqdef \frac{1}{\lambda n}X \alpha^t.\end{equation} Clearly, if $\alpha^t\to \alpha^*$, then $w^t \to w^*$.
We shall now design a method for maximizing the dual function $D$ and then in Theorem~\ref{thm:equiv} we claim that for quadratic loss functions, Algorithm~\ref{alg:PRIMAL} arises as an image, defined via   \eqref{eq:primal_from_dual},  of dual iterations of this  dual ascent method.

\paragraph{Design of the dual gradient ascent method.} Let $\xi(\alpha) \eqdef \tfrac{1}{2}\|X\alpha\|^2$. Since $\xi$ is a convex quadratic, we have
\[\xi(\alpha + h) = \xi(\alpha) + \langle \nabla \xi(\alpha), h\rangle + \frac{1}{2}h^T \nabla^2 \xi(\alpha)h,\\
\leq  \xi(\alpha) + \langle \nabla \xi(\alpha), h\rangle + \frac{\sigma}{2}\|h\|_B^2,
\]
where $\nabla \xi (\alpha)  = X^T X \alpha $ and $\nabla^2 \xi (\alpha) = X^T X $. Further, we define the block-diagonal matrix $B\eqdef Diag(X_1^T X_1, \dots,X_{K}^T X_K)$, and a norm associate with this matrix: \[\|h\|_B^2 \eqdef \sum_{k=1}^K \|X_k h_k\|^2.\] By $\sigma$ we refer to a large enough constant for which  $X^TX\preceq \sigma B$. In order to avoid unnecessary  technicalities, we shall assume that the matrices $X_k^T X_k$ are positive definite, which implies that $\|\cdot\|_B$ is a norm. It can be shown that $1\leq \sigma \leq K$. Clearly, $\xi$ is $\sigma$-smooth with respect to the norm $\|\cdot \|_B$. In view of the above, for all $h\in \R^n$ we can estimate $D$ from below as follows:
\begin{eqnarray*}
D(\alpha^t+h) &\geq &  -\frac{1}{\lambda n^2}\left(\xi (\alpha^t) + \langle \nabla \xi (\alpha^t), h \rangle + \frac{\sigma}{2}\sum_{k=1}^K \|X_k h_k\|^2\right) - \frac{1}{n}\sum_{i=1}^n \phi_i^*(-\alpha_i^t-h_i)\\
&=& -\frac{1}{\lambda n^2} \xi(\alpha^t) - \sum_{k=1}^K \left[ \frac{1}{\lambda n^2} \langle \nabla_k \xi (\alpha^t),h_k\rangle +  \frac{\sigma}{2\lambda n^2}\|X_k h_k\|^2 + \frac{1}{n}\sum_{i\in {\cal P}_k} \phi_i^*(-\alpha_i^t - h_i) \right] ,
\end{eqnarray*}
 where  $\nabla_k \xi (\alpha^t)$ corresponds to the subvector of $\nabla \xi(\alpha^t)$ formed by entries $i\in {\cal P}_k$. 

We  now let $h^t = (h^t_1,\dots,h^t_K)$  be the maximizer of this lower bound. Since the lower bound is separable in the blocks $\{h^t_k\}_k$, we can simply set
\begin{equation}\label{eq:DUAL-eq} h^t_k := \arg\min_{u\in \R^{|{\cal P}_k|}} \left\{D_k^t(u) \eqdef \frac{1}{\lambda n^2} \langle \nabla_k \xi (\alpha^{t}),u\rangle +  \frac{\sigma}{2\lambda n^2}\|X_k u\|^2 + \frac{1}{n}\sum_{i\in {\cal P}_k} \phi_i^*(-\alpha_i^t - u_i) \right\}.\end{equation}
Having computed $h_k^t$ for all $k$, we can set $\alpha_k^{t+1} = \alpha_k^t + h_k^t$ for all $k$, or equivalently, $\alpha^{t+1} = \alpha^t + h^t$.  This is formalized as Algorithm~\ref{alg:DUAL}.  Algorithm~\ref{alg:DUAL}  is a proximal gradient ascent method applied to the dual problem, with smoothness being measured using the block norm $\|h\|_B$. It is known that gradient ascent converges at a linear rate for smooth and strongly convex (for minimization problems) objectives.

\begin{algorithm}[t]
\begin{algorithmic}[1]
\State \textbf{Input:} $\sigma \in [1,K]$
\State \textbf{Choose:} $\alpha_k^0\in \R^{|{\cal P}_k|}$ for $k=1,2,\dots,K$
\For {$t = 0,1,2,\dots$}
	 \For {$k = 1$ \textbf{to} $K$} 
	 		\State	   $h^{t+1}_k =\arg\min_{u\in \R^{|{\cal P}_k|}}   D_k^t(u)$
	 		\Comment See \eqref{eq:DUAL-eq}
	 \EndFor
	 \State	 $\alpha^{t+1} = \alpha^t + h^t$
\EndFor	 	
\State \textbf{return} $w^t$
\end{algorithmic}
\caption{Dual Method}
\label{alg:DUAL}

\end{algorithm}

One of the main insights of this section is the following equivalence result.

\begin{theorem}[Equivalence of Algorithms \ref{alg:PRIMAL} and \ref{alg:DUAL} for Quadratic Loss] \label{thm:equiv} Consider the ridge regression problem. That is, set  $\phi_i(t) = \tfrac{1}{2}(t-y_i)^2$ for all $i$. Assume $\alpha_1^0, \dots, \alpha_K^0$ is chosen in the same way in Algorithms~\ref{alg:PRIMAL} and \ref{alg:DUAL}. Then the dual iterates $\alpha^t$ and the primal iterates $w^t$ produced by the two algorithms are related via \eqref{eq:primal_from_dual} for all $t\geq 0$. 
\end{theorem}

Since the dual method converges linearly, in view of the above theorem, so does the primal method. Here we only remark that the popular algorithm CoCoA+ \cite{ma2015distributed} arises if Step 5 in Algorithm~\ref{alg:DUAL} is done inexactly. Hence, we show that duality provides a deep relationship between the CoCoA+ and DANE algorithms, which were previously considered completely different.

\subsection{Proof of Theorem~\ref{thm:equiv} }

In this part we prove the theorem.

\paragraph{Primal and Dual Problems.} Since $\phi_i(t) = \tfrac{1}{2}(t-y_i)^2$, the primal problem \eqref{eq:ERM-linear-pred} is a ridge regression problem of the  form 
\begin{equation}\label{eq:P}\min_{w\in \R^d} f(w) = \frac{1}{2n}\|X^T w-y\|^2 + \frac{\lambda}{2}\|w\|^2,\end{equation}
where $X\in \R^{d\times n}$ and $y\in \R^n$. In view of \eqref{eq:dual_9878979s8}, the dual of \eqref{eq:P} is 
\begin{equation}\label{eq:D} \min_{\alpha\in \R^n} D(\alpha) =  \frac{1}{2\lambda n^2}\|X\alpha\|^2 + \frac{1}{2n}\|\alpha\|^2 -\frac{1}{n}y^T \alpha . \end{equation}

\paragraph{Primal Problem: Distributed Setup.} The primal objective function is of the form \eqref{eq:problem:distributed:simple}, where  in view of \eqref{eq:iuhd9898g9}, we have
$F_k(w) = \frac{K}{2n}\|X_k^T w - y_k\|^2 + \frac{\lambda}{2}\|w\|^2$. Therefore, 
\begin{equation} \label{eq:DF_k} \nabla F_k(w) = \frac{K}{n}X_k(X_k^T w - y_k) + \lambda w\end{equation}
and
$\nabla f(w) = \frac{1}{K}\sum_k \nabla F_k(w) = \frac{1}{K}\sum_k  \left(\frac{K}{n}X_k(X_k^T w - y_k) + \lambda w \right).$

\paragraph{Dual Method.}  Since $D$ is a quadratic, we have
\begin{eqnarray*}D(\alpha^t + h) &=& D(\alpha^t) + \nabla D(\alpha^t)^T h + \frac{1}{2}h^T \nabla^2 D(\alpha^t)h,\end{eqnarray*}
with \[\nabla D(\alpha^t)  = \frac{1}{\lambda n^2}X^T X \alpha^t + \frac{1}{n}(\alpha^t - y), \qquad \nabla^2 D(\alpha^t) = \frac{1}{\lambda n^2}X^T X + \frac{1}{n}I.\] 
We know that  $X^T X\preceq \sigma Diag(X_1^T X_1, \dots,X_{K}^T X_K)$. With this approximation, for all $h\in \R^n$ we can estimate $D$ from above by a node-separable quadratic function as follows:
\begin{eqnarray*}
D(\alpha^t+h) &\leq&  D(\alpha^t) + \left(\frac{1}{\lambda n^2}X^T X \alpha^t + \frac{1}{n}(\alpha^t - y)\right)^T h + \frac{1}{2n}\|h\|^2 + \frac{\sigma}{2\lambda n^2}\sum_{k=1}^K \|X_k h_k\|^2\\
&=& D(\alpha^t) + \frac{1}{n} \left[  \frac{1}{\lambda n}(X\alpha^t)^T X h + (\alpha^t - y)^T h + \frac{1}{2}\|h\|^2 + \frac{\sigma}{2\lambda n}\sum_{k=1}^K \|X_k h_k\|^2  \right]\\
&=&D(\alpha^t) + \frac{1}{n} \sum_{k=1}^K \left( (w^t)^T X_k h_k + (\alpha^t_k - y_k)^T h_k+ \frac{1}{2}\|h_k\|^2 + \frac{\sigma}{2\lambda n}\|X_k h_k\|^2  \right).
\end{eqnarray*}

Next, we shall define 
\begin{equation}\label{eq:xD}h_k^t \eqdef \arg \min_{h_k \in \R^{|{\cal P }_k|}}  \frac{\sigma}{2\lambda n}\|X_k h_k\|^2     + \frac{1}{2}\|h_k\|^2 - ( y_k - X_k^T w^t - \alpha^t_k )^T h_k \end{equation}
for $k=1,2,\dots,K$ and then set \begin{equation}
\label{eq:xxx}\alpha^{t+1} = \alpha^t + h^t.\end{equation}

\paragraph{Primal Version of the Dual Method.} Note that \eqref{eq:xD} has the same form as \eqref{eq:D}, with $X$ replaced by $X_k$, $\lambda$ replaced by $\lambda/\sigma$ and $y$ replaced by $c_k:=y_k - X_k^T w^t - \alpha^t_k$. Hence, we know that 
\begin{equation}\label{eq:s}s_k^t \eqdef \frac{1}{(\lambda/\sigma ) n}X_k h_k^t\end{equation}
is the optimal solution of the primal problem of \eqref{eq:xP}:
\begin{equation}\label{eq:xP}s_k^t =\arg \min_{s\in \R^d}  \frac{1}{2n}\|X_k^T s-c_k\|^2 + \frac{\lambda/\sigma}{2}\|s\|^2.\end{equation}

Hence, the primal version of  method \eqref{eq:xxx} is given by
\begin{eqnarray*}w^{t+1}&\overset{\eqref{eq:primal_from_dual}}{=}& \frac{1}{\lambda n}X \alpha^{t+1} \overset{\eqref{eq:xxx}}{=}\frac{1}{\lambda n} X (\alpha^t + h^t) \overset{\eqref{eq:primal_from_dual}}{=} w^t + \frac{1}{\lambda n} \sum_{k=1}^K X_k h^t_k\\
& = &\frac{1}{K} \sum_{k=1}^K \left( w^t + \frac{K}{\sigma}\frac{\sigma}{\lambda n}  X_k h^t_k \right) \overset{\eqref{eq:s}}{=} \frac{1}{K} \sum_{k=1}^K \left( w^t + \frac{K}{\sigma}s_k^t \right).
\end{eqnarray*}

With the change of variables $w := w^t + \frac{K}{\sigma}s$ (i.e., $s = \frac{\sigma}{K}(w-w^t)$), from \eqref{eq:xP} we know that $w_k^{t+1}:=w^t + \frac{K}{\sigma}s_k^t$ solves
\begin{equation}\label{eq:xPnew}w_k^{t+1} =\arg \min_{w\in \R^d} \left\{L_k(w) \eqdef \frac{1}{2n}\left\|X_k^T \frac{\sigma}{K}(w-w^t)-c_k \right\|^2 + \frac{\lambda/\sigma}{2} \left\|\frac{\sigma}{K}(w-w^t) \right\|^2 \right\}\end{equation}
and $w^{t+1} = \frac{1}{K}\sum_{k=1}^K w_k^{t+1}$.

Let us now rewrite the function in \eqref{eq:xPnew} so as to connect it to Algorithm~\ref{alg:PRIMAL}:
\begin{eqnarray*} L_{k}(w) &=& \frac{1}{2n}\left\|X_k^T \frac{\sigma}{K}(w-w^t)-c_k \right\|^2 + \frac{\lambda/\sigma}{2} \left\|\frac{\sigma}{K}(w-w^t) \right\|^2\\
&=& \frac{1}{2n}\frac{\sigma^2}{K^2}\left\|  (X_k^T w-y_k)  - \underbrace{\left(X_k^T w^t -y_k + \frac{K}{\sigma} c_k\right)}_{d_k} \right\|^2 + \frac{\lambda \sigma^2}{2K^3}\|w\|^2 - \frac{\lambda \sigma^2}{2K^3}\|w\|^2 + \frac{\lambda/\sigma}{2} \left\|\frac{\sigma}{K}(w-w^t) \right\|^2\\
&=& \frac{1}{2n}\frac{\sigma^2}{K^2} \left( \left\| X_k^T w - y_k\right\|^2 +  \|d_k\|^2 - 2 (X_k^T w - y_k)^T d_k\right) + \frac{\lambda \sigma^2}{2K^3}\|w\|^2 - \frac{\lambda \sigma^2}{2K^3}\|w\|^2 + \frac{\lambda\sigma}{2 K^2} \left\| w-w^t \right\|^2\\
&=&\frac{\sigma^2}{K^3} \left(\frac{K}{2n}  \left\| X_k^T w - y_k\right\|^2 +  \frac{K}{2n}\|d_k\|^2 - \frac{K}{n} (X_k^T w - y_k)^T d_k\right) + \frac{\lambda \sigma^2}{2K^3}\|w\|^2 \\
&& \qquad - \frac{\lambda \sigma^2}{2K^3}\|w\|^2 + \frac{\lambda\sigma}{2 K^2} \left\| w-w^t \right\|^2\\ 
&=&\frac{\sigma^2}{K^3} \underbrace{\left(\frac{K}{2n}  \left\| X_k^T w - y_k\right\|^2 + \frac{\lambda }{2}\|w\|^2 \right)}_{F_k(w)} + \frac{\sigma^2}{K^3} \left(  \frac{K}{2n}\|d_k\|^2 - \frac{K}{n} (X_k^T w - y_k)^T d_k\right)  \\
&& \qquad - \frac{\lambda \sigma^2}{2K^3}\|w\|^2 + \frac{\lambda\sigma}{2 K^2} \left\| w-w^t \right\|^2\\
&=&\frac{\sigma^2}{K^3} F_k(w) - \frac{\sigma^2}{K^2 n}  (X_k^T w - y_k)^T d_k + \frac{\sigma^2}{2n K^2} \|d_k\|^2  - \frac{\lambda \sigma^2}{2K^3}\|w\|^2 + \frac{\lambda\sigma}{2 K^2} \left\| w-w^t \right\|^2\\
&=&\frac{\sigma^2}{K^3} F_k(w) - \frac{\sigma^2}{K^2 n}  (X_k d_k)^T w     - \frac{\lambda \sigma^2}{2K^3}\|w\|^2 + \frac{\lambda\sigma}{2 K^2} \left\| w-w^t \right\|^2 + \underbrace{\left( \frac{\sigma^2}{2n K^2} \|d_k\|^2+  \frac{\sigma^2}{K^2 n}  y_k^T d_k\right)}_{\beta_1}.
\end{eqnarray*}

Next, since $\|w\|^2 = \|w-w^t\|^2 - \|w^t\|^2 + 2(w^t)^T w$, we can further write 

\begin{eqnarray*} L_{k}(w) &=& \frac{\sigma^2}{K^3} F_k(w) - \frac{\sigma^2}{K^2 n}  (X_k d_k)^T w     - \frac{\lambda \sigma^2}{2K^3} ( \|w-w^t\|^2 - \|w^t\|^2 + 2(w^t)^T w) + \frac{\lambda\sigma}{2 K^2} \left\| w-w^t \right\|^2 + \beta_1\\
&=&\frac{\sigma^2}{K^3} F_k(w) - \frac{\sigma^2}{K^2 n}  (X_k d_k)^T w  - \frac{\lambda \sigma^2}{K^3}   (w^t)^T w + \left(\frac{\lambda\sigma}{2 K^2} - \frac{\lambda \sigma^2}{2K^3} \right) \left\| w-w^t \right\|^2     + \underbrace{ \frac{\lambda \sigma^2}{2K^3} \|w^t\|^2 + \beta_1}_{\beta_2}\\
&=& \frac{\sigma^2}{K^3} \left(F_k(w) - \left(\frac{K}{n} X_k d_k + \lambda w^t \right)^T w  + \frac{\lambda}{2}\left(\frac{K}{\sigma}-1\right)\|w-w^t\|^2 \right) + \beta_2\\
&=& \frac{\sigma^2}{K^3} \left(F_k(w) - \left(\nabla F_k(w^t) - \frac{K^2}{\sigma n} X_k ( X_k^T w^t -y_k + \alpha_k^t)\right)^T w  + \frac{\mu}{2}\|w-w^t\|^2 \right) + \beta_2\\
&=& \frac{\sigma^2}{K^3} \left(F_k(w) - \left(\nabla F_k(w^t) - \frac{ K}{\sigma}\underbrace{\frac{K}{ n} X_k ( X_k^T w^t -y_k + \alpha_k^t)}_{z^t_k}\right)^T w  + \frac{\mu}{2}\|w-w^t\|^2 \right) + \beta_2\\
&=& \frac{\sigma^2}{K^3} \left(F_k(w) - \left(\nabla F_k(w^t) - (\eta \nabla F_k(w^t) + g_k^t) \right)^T w  + \frac{\mu}{2}\|w-w^t\|^2 \right) + \beta_2,
\end{eqnarray*}
where the last step follows from the claim that  $\eta z_k^t = \eta \nabla F_k(w^t) + g_k^t$. We now prove the claim. First, we have
\begin{eqnarray*}
\eta z_k^t &=& \eta \frac{K}{ n} X_k ( X_k^T w^t -y_k + \alpha_k^t)\\
&=& \eta \frac{K}{ n} X_k ( X_k^T w^t -y_k) + \eta \frac{K}{ n} X_k  \alpha_k^t\\
&=& \eta \left(\frac{K}{ n} X_k ( X_k^T w^t -y_k) + \lambda w^t \right) + \eta \left(\frac{K}{ n} X_k  \alpha_k^t - \lambda w^t \right) \\
&\overset{\eqref{eq:DF_k}}{=} & \eta \nabla F_k(w^t) +  \eta \left(\frac{K}{ n} X_k  \alpha_k^t - \lambda w^t \right).
\end{eqnarray*}
Due to the definition of $g_k^0$ in Step 5 of Algorithm~\ref{alg:PRIMAL} as $g_k^0 = \eta (\frac{K}{n} X_k \alpha_k^0 - \lambda w^0)$, we observe that the claim holds for $t=0$. If we show that 
\[g_k^{t} =  \eta \left(\frac{K}{ n} X_k  \alpha_k^{t} - \lambda w^{t} \right)\]
for all $t\geq 0$, then we are done.  This can be shown by induction.  This finishes the proof of  Theorem~\ref{thm:equiv}.

\end{document}